\documentclass[10pt]{article}
\usepackage{amsmath,amssymb,bm}
\usepackage{amsthm}
\usepackage{mathtools}
\usepackage[noend]{algorithmic}
\usepackage[ruled,vlined]{algorithm2e}
\usepackage{url}
\usepackage{makeidx}
\usepackage{booktabs,multirow,threeparttable}

\usepackage{graphicx,float,psfrag,epsfig,caption}
\usepackage[usenames,dvipsnames,svgnames,table]{xcolor}
\definecolor{darkgreen}{rgb}{0.0,0,0.9}
\usepackage[pagebackref,letterpaper=true,colorlinks=true,pdfpagemode=none,citecolor=OliveGreen,linkcolor=BrickRed,urlcolor=BrickRed,pdfstartview=FitH]{hyperref}
\usepackage{epstopdf}
\usepackage{color}
\usepackage{xr}
\usepackage{subfigure}
\usepackage{caption}
\usepackage{graphicx}
\usepackage[utf8]{inputenc}
\usepackage{tikz}
\usepackage[english]{babel}
\usepackage{am}
\newtheorem{case}{Case}

\usepackage{scalerel,stackengine}

\usepackage[mathscr]{euscript}

\DeclareSymbolFont{rsfs}{U}{rsfs}{m}{n}
\DeclareSymbolFontAlphabet{\mathscrsfs}{rsfs}

\stackMath
\newcommand\reallywidehat[1]{%
\savestack{\tmpbox}{\stretchto{%
  \scaleto{%
    \scalerel*[\widthof{\ensuremath{#1}}]{\kern.1pt\mathchar"0362\kern.1pt}%
    {\rule{0ex}{\textheight}}
  }{\textheight}%
}{2.4ex}}%
\stackon[-6.9pt]{#1}{\tmpbox}%
}
\usepackage[mathscr]{euscript}

\DeclareSymbolFont{rsfs}{U}{rsfs}{m}{n}
\DeclareSymbolFontAlphabet{\mathscrsfs}{rsfs}

\numberwithin{equation}{section}

\newtheoremstyle{myexample} 
    {\topsep}                    
    {\topsep}                    
    {\rm }                   
    {}                           
    {\bf }                   
    {.}                          
    {.5em}                       
    {}  

\makeatletter

\newcommand*{\rom}[1]{\expandafter\@slowromancap\romannumeral #1@}
\makeatother

\usetikzlibrary{arrows,shapes}
\begin{document}

\title{\bf Analysis of Regularized Least-Squares in Reproducing Kernel Kre\u{\i}n Spaces}

\author{Fanghui Liu\thanks{Department of Electrical Engineering
		(ESAT-STADIUS), KU Leuven.~~~~ E-mail: fanghui.liu@kuleuven.be} \;\;\;\; Lei Shi\thanks{Shanghai Key Laboratory for Contemporary Applied Mathematics, School of Mathematical Sciences, Fudan University} \;\;\;\; Xiaolin Huang\thanks{Institute of Image Processing and Pattern Recognition, Institute of Medical Robotics, Shanghai Jiao Tong University} \;\;\;\; Jie Yang\footnotemark[3] \;\;\;\;
	Johan A.K. Suykens\footnotemark[1]}

\maketitle

\begin{abstract}
	\noindent In this paper, we study the asymptotic properties of regularized least squares with indefinite kernels in reproducing kernel Kre\u{\i}n spaces (RKKS). By introducing a bounded hyper-sphere constraint to such non-convex regularized risk minimization problem, we theoretically demonstrate that this problem has a globally optimal solution with a closed form on the sphere, which makes approximation analysis feasible in RKKS. Regarding to the original regularizer induced by the indefinite inner product, we modify traditional error decomposition techniques, prove convergence results for the introduced hypothesis error based on matrix perturbation theory, and derive learning rates of such regularized regression problem in RKKS. Under some conditions, the derived learning rates in RKKS are the same as that in reproducing kernel Hilbert spaces (RKHS), which is actually the first work on approximation analysis of regularized learning algorithms in RKKS.
\end{abstract}


\section{Introduction}
Kernel methods \cite{Sch2003Learning,suykens2002least,liu2020learning} have demonstrated success in statistical learning, such as classification \cite{Zhu2002Kernel,shang2019unsupervised}, regression \cite{shi2019sparse,farooq2019learning}, and clustering \cite{dhillon2004kernel,terada2019kernel,liu2020simplemkkm}.
The key ingredient of kernel methods is a kernel function, that is positive definite (RD) and can be associated with the inner product of two vectors in a reproducing kernel Hilbert space (RKHS).
Nevertheless, in real-world applications, the used kernels might be \emph{indefinite} (real, symmetric, but not positive definite) \cite{Ying2009Analysis,Ga2016Learning,oglic2019scalable} due to \emph{intrinsic} and \emph{extrinsic} factors.
Here, \emph{intrinsic} means that we often meet some indefinite kernels such as \textit{tanh} kernel \cite{smola2001regularization}, \textit{TL1} kernel \cite{huang2017classification}, \textit{log} kernel \cite{boughorbel2005conditionally}, and hyperbolic kernel \cite{cho2019large}. 
Meanwhile, \emph{extrinsic} indicates that some positive definite kernels degenerate to indefinite ones in some cases.
An intuitive example is that a linear combination of PD kernels (with negative coefficient) \cite{Cheng2005Learning} is an indefinite kernel.
Polynomial kernels on the unit sphere are not always PD \cite{pennington2015spherical}.
In manifold learning, the Gaussian kernel with some geodesic distances would lead to be an indefinite one.
In neural networks, the sigmoid kernel with various values of hyper-parameters are mostly indefinite \cite{Cheng2004Learning}.
We refer to a survey \cite{Schleif2015Indefinite} for details.

Efforts on indefinite kernels are often based on a reproducing kernel Kre\u{\i}n space (RKKS) \cite{Cheng2004Learning,Ga2016Learning,Alabdulmohsin2016Large,saha2020learning} which is endowed by the indefinite inner product and is associated with a (reproducing) indefinite kernel.
The related optimization problem is often non-convex due to the non-positive definiteness of the used indefinite kernel.
{Since the indefinite inner product in RKKS does not define a norm, typical empirical risk minimization that minimizes over a class of functionals in RKHS is transformed to \emph{stabilization} in RKKS via projection \cite{ando2009projections}. 
	Here \emph{stabilization} means finding a stationary point (more precisely, saddle points) instead of a minimum, as conducted by a series of indefinite kernel based algorithms \cite{Cheng2004Learning,Ga2016Learning,saha2020learning}.
	In stabilization, the indefinite inner product in RKKS via a projection view can be still served as a valid regularization mechanism \cite{Ga2016Learning}.
	Although stabilization based algorithms achieve promising performance and nice theoretical results, Oglic and G\"artner \cite{oglic18a,oglic2019scalable} directly consider empirical risk minimization in RKKS restricted in a hyper-sphere, and demonstrate that this setting is able to generalize well.}

In learning theory, the asymptotic behavior of these regularized indefinite kernel learning based algorithms in RKKS has not been fully investigated in an approximation theory view.
Current literature \cite{Wu2006Learning,Steinwart2009Optimal,lin2017distributed,jun2019kernel} on approximation analysis often focus on regularized methods in RKHS, but their results could not be directly applied to that in RKKS due to the following two reasons.
{First, approximation analysis in RKHS often requires a (globally) optimal solution yielded by learning algorithms.
	While the target of most indefinite kernel based methods via stabilization in RKKS is to seek for a saddle point instead of a minimum. In this case, traditional concentration estimates could be invalid to that in RKKS.}
Second, in RKKS, the regularizer endowed by the indefinite inner product might be negative, which would fail to quantify complexity of a hypothesis. The classical error decomposition technique \cite{cucker2007learning,lin2017distributed} might be infeasible to our setting in RKKS.

To overcome the mentioned essential problems, in this paper, we study learning rates of least squares regularized  regression in RKKS.
Motivated by \cite{oglic18a}, we focus on a typical empirical risk minimization in RKKS, i.e., indefinite kernel ridge regression in a hyper-sphere region endowed by the indefinite inner product. For this purpose, we provide a detailed error analysis and then derive learning rates.
To be specific, in algorithm, we demonstrate that, the solution to our considered kernel ridge regression model in RKKS with a spherical constraint can be achieved on the hyper-sphere. Subsequently, albeit non-convex, this model admits a global minimum with a closed form as demonstrated by \cite{oglic18a}.
We start the analysis from the regularized algorithm that has an analytical solution and obtain the first-step to understand the learning behavior in RKKS.
In theory, we modify the traditional error decomposition approach, and thus the excess error can be bounded by the sample error, the regularization error, and the additional hypothesis error. We provide estimates for the introduced hypothesis error based on matrix perturbation theory for non-Hermitian and non-diagonalizable matrices and then derive convergence rates of such model. 
Our analysis is able to bridge the gap between the least squares regularized regression problem in RKHS and RKKS.
Under some conditions, the derived learning rates in RKKS is the same as that in RKHS (the best case).
To the best of our knowledge, this is the first work to study learning rates of regularized risk minimization in RKKS.

The rest of the paper is organized as follows.
We briefly introduce the basic concepts of Kre\u{\i}n spaces and RKKS in Section~\ref{sec:pre}.
Section~\ref{sec:prosetting} presents the problem setting and main results under some fair assumptions.
In Section~\ref{sec:algo}, we present the least squares regularized regression model in RKKS and give a globally optimal solution to aid the proof.
In Section~\ref{sec:prooframe}, we give the framework of convergence analysis for the modified error decomposition technique, detail the estimates for the introduced hypothesis error, and derive the learning rates.
In Section~\ref{sec:exp}, we report numerical experiments to demonstrate our theoretical results and the conclusion is drawn in Section~\ref{sec:conclusion}.

\section{Preliminaries}
\label{sec:pre}

In this section, we briefly introduce the definitions and basic properties of Kre\u{\i}n spaces and the reproducing kernel Kre\u{\i}n space (RKKS) that we shall need later.
Detailed expositions can be found in the book \cite{bognar1974indefinite}.

We begin with a vector space $\mathcal{H_K}$ defined on the scalar field $\mathbb{R}$. 
\begin{definition}\label{def1}
	{(inner product space) An inner product space is a vector space  $\mathcal{H_K}$ defined on the scalar field $\mathbb{R}$ together with a bilinear form $\langle \cdot, \cdot \rangle_{\mathcal{H_K}}$ called inner product that satisfies the following conditions\\
		\textit{i)} symmetry: $\forall f,g \in \mathcal{H_K}$, we have $\langle f, g \rangle_{\mathcal{H_K}} = \langle g,f \rangle_{\mathcal{H_K}}$.\\
		\textit{ii)} linearity:$\forall f,g,h \in \mathcal{H_K}$ and two scalrs $a,b \in \mathbb{R}$, we have $\langle af+bg, h \rangle_{\mathcal{H_K}} = a\langle f, h \rangle_{\mathcal{H_K}} +  b\langle g, h \rangle_{\mathcal{H_K}} $.\\
		\textit{iii)} non-degenerate: for $f \in \mathcal{H_K}$, if $\langle f, g \rangle_{\mathcal{H_K}} = 0 $ for all $g \in \mathcal{H_K}$ implies that $f=0$.}
\end{definition}

{If $\langle f, f \rangle_{\mathcal{H_K}} > 0 $ holds for any $f \in \mathcal{H_K}$ with $f \neq 0$, then the inner product on $\mathcal{H_K}$ is \emph{positive}.
	If there exists $f,g \in \mathcal{H_K}$ such that $\langle f, f \rangle_{\mathcal{H_K}} > 0 $ and $\langle g, g \rangle_{\mathcal{H_K}} < 0 $, then the inner product is called \emph{indefinite}, and $\mathcal{H_K}$ is an indefinite inner product space.
	Recall that Hilbert spaces satisfy the above conditions and admit the positive inner product.
	After reviewing the indefinite inner product, we are ready to introduce the definition of Kre\u{\i}n space.}

\begin{definition}\label{definiterkks}
	(Kre\u{\i}n space \cite{bognar1974indefinite}) The vector space $\mathcal{H_K}$ with the inner product $\langle \cdot, \cdot \rangle_{\mathcal{H_K}}$ is a Kre\u{\i}n space if there exist two Hilbert spaces $\mathcal{H}_+$ and $\mathcal{H}_-$ such that\\
	\textit{i)} the vector space $\mathcal{H_K}$ admits a direct orthogonal sum decomposition $\mathcal{H_K} = \mathcal{H}_+ \oplus \mathcal{H}_-$.\\
	\textit{ii)} all $f \in \mathcal{H_K}$ can be decomposed into $f=f_++f_-$, where $f_+ \in \mathcal{H}_+$ and $f_- \in \mathcal{H}_-$, respectively.\\
	\textit{iii)} $\forall f,g \in \mathcal{H_K}$, $\langle f,g\rangle_{\mathcal{H_K}}=\langle f_+,g_+\rangle_{\mathcal{H}_+} - \langle f_-,g_-\rangle_{\mathcal{H}_-}$.
\end{definition}
From the definition, the decomposition $\mathcal{H_K} = \mathcal{H}_+ \oplus \mathcal{H}_-$ is not necessarily unique. For a fixed decomposition, the inner product $ \langle f, g \rangle_{\mathcal{H_K}}$ is given accordingly \cite{Ga2016Learning,oglic18a}.
Kre\u{\i}n spaces are indefinite inner product spaces endowed with a Hilbertian topology.
The key difference with Hilbert spaces is that the inner products might be negative for Kre\u{\i}n spaces, i.e., there exists $f \in \mathcal{H_K}$ such that $\langle f,f\rangle_{\mathcal{H_K}} < 0$.
If $\mathcal{H}_+$ and $\mathcal{H}_-$ are two RKHSs, the Kre\u{\i}n space $\mathcal{H_K}$ is a RKKS associated with a unique indefinite reproducing kernel $k$ such that the reproducing property holds, i.e., $\forall f \in \mathcal{H_K},~f(x) = \langle f,k(x,\cdot) \rangle_{\mathcal{H_K}}$.
\begin{proposition}
	(positive decomposition \cite{bognar1974indefinite})
	An indefinite kernel $k$ associated with a RKKS admits a positive decomposition $k = k_+ - k_-$, with two positive definite kernels $k_+$ and $k_-$.
	\vspace{-0.1cm}
\end{proposition}
Typical examples include a wide range of commonly used indefinite kernels, such as a linear combination of PD kernels \cite{Cheng2005Learning}, and conditionally PD kernels \cite{schaback1999native,wendland2004scattered}.
It is important to note that, not every indefinite kernel function admits such positive decomposition as a difference between two positive definite kernels.
Nevertheless, this can be conducted on finite discrete spaces, e.g., eigenvalue decomposition of indefinite kernel matrices.
{In fact, for any given an indefinite kernel, whether it can be associated with RKKS still remains a long-lasting open question.
	For example, the hyperbolic kernel \cite{cho2019large} is based on the hyperboloid model \cite{sala2018representation}, in which the distance between two point is defined as the length of the geodesic path on the hyperboloid that connects the two points. 
	Although the used hyperboloid space stems from a finite-dimensional Kre\u{\i}n space, it is unclear whether the derived hyperbolic kernel is associated with RKKS or not.
	Besides, the existence of positive decomposition for the \textit{TL1} kernel \cite{huang2017classification}, defined by the truncated $\ell_1$ distance, is also unknown.
	Our results in this paper are based on RKKS, and can be applied to these kernels if they can be associated with RKKS.
}
%
\begin{definition}\label{defassrkks}
	(Associated RKHS of RKKS \cite{Cheng2004Learning})
	Let ${\mathcal{H_K}}$ be a RKKS with the direct orthogonal sum decomposition into two RKHSs ${\mathcal{H}_{+}}$ and ${\mathcal{H}_{-}}$.
	Then the associated RKHS ${\mathcal{H_{\bar K}}}$ endowed by ${\mathcal{H_K}}$ is defined with the positive inner product
	\begin{equation*} \langle f,g\rangle_{{\mathcal{H_{\bar K}}}}=\langle f_+,g_+\rangle_{{\mathcal{H_{+}}}} +\langle f_-,g_-\rangle_{{\mathcal{H_{-}}}},~~\forall f,g \in {\mathcal{H_K}}\,.
	\end{equation*}
\end{definition}
Note that ${\mathcal{H_{\bar K}}}$ is the smallest Hilbert space majorizing the RKKS ${\mathcal{H_{K}}}$ with $| \langle f,f\rangle_{{\mathcal{H_K}}} | \leq \| f\|_{{\mathcal{H_{\bar K}}}}^2 = \| f_+\|_{\mathcal{H_+}}^2 + \| f_-\|_{\mathcal{H_-}}^2$.
Denote $C(X)$ as the space of continuous functions on $X$ with the norm $\| \cdot \|_{\infty}$, and suppose that $\kappa \coloneqq  \sqrt{2}\sup_{\bm x \in X}  \sqrt{k_+(\bm x, \bm x) + k_-(\bm x, \bm x')} < \infty $. The reproducing property in RKKS indicates that $\forall f \in \mathcal{H_K}$, we have
\begin{equation}\label{fnormdiff}
\|f\|_{\infty} = \sup_{\bm x \in X} \big| \big \langle f, k(\bm x,\cdot) \big \rangle \big|  \leq \kappa \| f\|_{\mathcal{H_{\bar K}}}  \,.
\end{equation}
\begin{definition}\label{defemkernel}
	(The empirical covariance operator in RKKS \cite{Pekalska2009Kernel})
	Let $k$ be an indefinite kernel associated with a RKKS ${\mathcal{H_K}}$, $\psi: X \rightarrow \mathcal{H_K}$ be a mapping of the data in ${\mathcal{H_K}}$ and $\bm \Psi = [\psi(\bm x_1), \psi(\bm x_2), \cdots, \psi(\bm x_m)]$ be a sequence of images of the training data in ${\mathcal{H_K}}$, then its empirical non-centered covariance operator $T: {\mathcal{H_K}} \rightarrow {\mathcal{H_K}}$ is defined by
	\begin{equation}\label{emkernel}
	T = \frac{1}{m} \bm \Psi \bm \Psi^{*} \,,
	\end{equation}
	which is not positive definite in the Hilbert sense, but it is in the Kre\u{\i}n sense satisfying  $\langle \zeta, T \zeta \rangle_{\mathcal{H_K}} \geq 0$ for $\zeta \neq 0$.
\end{definition}
The operator $T$ actually depends on the sample set and it can be linked to an empirical kernel \cite{Guo2017Optimal}.
In our paper, we choose the mapping $\psi(\bm x) : = k(\bm x, \cdot)$ to obtain the empirical covariance operator $T$. 
Since $\langle f, T f \rangle_{\mathcal{H_K}}$ is nonnegative, we use it as a regularizer to aid our proof.

\section{Problem Setting and Main Results}
\label{sec:prosetting}

In this section, we present our problem setting and give its convergence rate under some fair assumptions.

\subsection{Problem setting}
Let $X$ be a compact metric space and $Y \subseteq \mathbb{R}$, we assume that a sample set $\bm z = \{  (\bm x_i, y_i) \}_{i=1}^m \in Z^m $ is drawn from a non-degenerate Borel probability measure $\rho$ on $X \times Y$.
In the context of statistical learning theory, the \emph{target function} of $\rho$ is defined by
$f_{\rho}(\bm x) = \int_Y y \mathrm{d} \rho(y|\bm x), \bm x \in X$,
where $\rho(\cdot|\bm x)$ is the conditional distribution of $\rho$ at $\bm x \in X$.
The indefinite kernel function $k: X \times X \rightarrow \mathbb{R}$ is endowed by the RKKS $\mathcal{H_K}$.
The associated indefinite kernel matrix is given by $\bm K = [k(\bm x_i, \bm x_j)]_{i,j=1}^m$ on the sample set.
The goal of a supervised learning task in RKKS endowed by $k$ is to find a hypothesis $f: X \rightarrow Y$ such that $f(\bm x)$ is a good approximation of the label $y \in Y$ corresponding to a new instance $\bm x \in X$.

Motivated by \cite{oglic18a}, we consider the least squares regularized regression problem in a bounded region induced by the original regularization mechanism of RKKS
\begin{equation}\label{fzrs}
f_{\bm{z},\lambda} \coloneqq  \argmin_{f \in\mathcal{B} (r)} \left\{ \frac{1}{m} \sum_{i=1}^{m} \big(f(\bm x_i) - y_i \big)^2 + \lambda \langle f,f\rangle_{\mathcal{H_K}} \right\}\,,
\end{equation}
where $\mathcal{B} (r)$ is spanned by the training data $\{ \bm x_i \}_{i=1}^m$ in $\mathcal{H_K}$ in a bounded hyper-sphere
	\begin{equation*}
		\mathcal{B} (r) \coloneqq  \left\{ f \in \operatorname{span}\{ k(\bm x_1,\cdot), k(\bm x_2,\cdot), \cdots, k(\bm x_m,\cdot) \}: \frac{1}{m} \sum_{i=1}^{m} \left( f(\bm x_i)  \right)^2 \leq r^2 \right\}\,.
	\end{equation*}
Here we employ the original regularization mechanism of RKKS, which aims to understand the learning behavior in RKKS and avoid the inconsistency when using various regularizers spanned by different spaces.
Our result can be applied to other settings with different regularizers in the next description.
Following with \cite{oglic18a}, we consider a risk minimization problem in a hyper-sphere instead of stabilization.
The considered hyper-spherical constraint in RKKS is able to prohibit the objective function value in problem~\eqref{fzrs} approaches to infinity, avoiding a meaningless solution.
The radius $r$ can be chosen by cross validation or hyper-parameter optimization \cite{oglic18a} in practice and is naturally needed and common in classical approximation analysis in RKHS \cite{Wu2006Learning,cucker2007learning,Steinwart2009Optimal}.
Such risk minimization problem still preserves the specifics of learning in RKKS, i.e., there exists some points $ f \in \mathcal{B} (r)$ such that $\langle f,f\rangle_{\mathcal{H_K}} < 0$, and generalizes well when compared to stabilization, as indicated by \cite{oglic18a}.
{One main reason why we consider the risk minimization problem is that, stabilization in RKKS does not necessarily have a unique saddle point, which makes approximation analysis infeasible to define the concentration of certain empirical hypotheses around some target hypothesis.
	Conversely, the studied risk minimization in a hyper-sphere, problem~\ref{fzrs}, leads to a (globally) optimal solution.}
This nice result on minimum motivates us to obtain the first-step to understand the learning behavior in RKKS.

\subsection{Main results}
\label{sec:appro}
In this section, we state and discuss our main results.
To illustrate our analysis, we need the following notations and assumptions.

In the least squares regression problem, the expected (quadratic) risk is defined as $\mathcal{E}(f) = \int_Z (f(\bm x) - y)^2 \mathrm{d} \rho$.
The empirical risk functional is defined on the sample $\bm z$, i.e., $\mathcal{E}_{\bm z}(f) = \frac{1}{m} \sum_{i=1}^{m} \big(f( \bm x_i) - y_i \big)^2$.
To measure the estimation quality of $f_{\bm z, \lambda}$, one natural way in approximation theory is the \emph{excess risk}: 
$\mathcal{E}(f_{\bm z, \lambda}) - \mathcal{E}(f_{\rho})$.

\begin{assumption}\label{assrho}
	(Existence and boundedness of $f_{\rho}$) we assume that the \emph{target function} $f_{\rho} \in \mathcal{H_K}$ exists and bounded. There exits a constant $M^* \geq 1$, such that
	\begin{equation*}\label{assumption1}
	|f_{\rho}|\leq M^* \mbox{ for almost $ \bm x \in
		{X}$ with respect to $\rho_{X}$}\,.
	\end{equation*}
\end{assumption}
\noindent{\bf Remark:} This is a standard assumption in approximation analysis \cite{cucker2007learning,lin2017distributed,Rudi2017Generalization}.
{Here we remark that existence of $f_{\rho}$ implies a bounded hyper-sphere region is needed, e.g., the used radius $r$ in problem~\eqref{fzrs}.}
In fact, the existence of $f_{\rho}$ is not ensured if we consider a potentially infinite dimensional RKKS $\mathcal{H_K}$, possibly universal \cite{Steinwart2008SVM}.
Instead, in approximation analysis, the infinite dimensional RKKS is substituted by a finite one, i.e., $\mathcal{H}_{\mathcal{K}}^r = \{ f \in \mathcal{H_K}: \| f \| \leq r \}$ with $r$ fixed a priori, where the norm $\| f \|$ is defined in some associated Hilbert spaces, e.g., ${\mathcal{H_{\bar K}}}$ or using the non-negative inner product $\langle f, T f \rangle_{\mathcal{H_K}}$ by the empirical covariance operator.
In this case, a minimizer of risk $\mathcal{E}$ always exists but $r$ is fixed with a prior and $\mathcal{H}_{\mathcal{K}}^r$ cannot be universal.
As a result, assuming the existence of $f_{\rho}$ implies that $f_{\rho}$ belongs to a ball of radius $r_{\rho, \mathcal{H_K}}$.
So this is the reason why the spherical constraint is indeed taken into account in approximation analysis.

For a tighter bound, we need the following \emph{projection operator}.
\begin{definition}\label{proj}
	(projection operator \cite{Chen2004Support})
	For $B > 0$, the projection operator $\pi \coloneqq  \pi_{B}$ is defined on the space of measurable functions $f: X \rightarrow \mathbb{R}$ as
	\begin{equation*}\label{BBPdef}
	\pi_B(f)(\bm x)= \left\{
	\begin{array}{rcl}
	\begin{split}
	& B,  ~~\text{if}~~ f(\bm x) > B ; \\
	&-B, ~~\text{if}~~f(\bm x) < -B ; \\
	& f(\bm x), ~~\text{if}~~-B \leq f(\bm x) \leq B\,,
	\end{split}
	\end{array} \right.
	\end{equation*}
	and then the projection of $f$ is denoted as $\pi_B(f)(\bm x) = \pi_B(f(\bm x))$.
\end{definition}
The projection operator is beneficial to the $\| \cdot \|_{\infty}$-bounds for sharp estimation.
Besides, we consider the standard output assumption\footnote{For unbounded outputs, the moment hypothesis \cite{Wang2011Optimal} is suitable but the introduced hypothesis error in our analysis depends on the standard output assumption.} $|y| \leq M$, and then we have
$\mathcal{E}_{\bm z}\big(\pi_B(f_{\bm{z},\lambda})\big)  \leq \mathcal{E}_{\bm z}\big(f_{\bm{z},\lambda}\big)$.
So it is more accurate to estimate $f_{\rho}$ by $\pi_{M^*}(f_{\bm{z},\lambda})$ instead of $f_{\bm{z},\lambda}$.
Therefore, our approximation analysis attempts to bound the error $\| \pi_{M^*}(f_{\bm{z},\lambda})  - f_{\rho} \|^2_{L_{\rho_X}^{p^*}}$ in the space ${L_{\rho_X}^{p^*}}$ with some $p^*>0$, where $L^{p^*}_{\rho_{{X}}}$ is a weighted $L^{p^*}$-space with the norm $\|f\|_{L^{p^*}_{\rho_{{X}}}} = \Big( \int_{{X}} |f(\bm x)|^{p^*} \mathrm{d} \rho_{X}(\bm x) \Big)^{1/{p^*}} $.
Specifically, in our analysis, the excess error is exactly the distance in $L_{\rho_X}^{2}$ due to the strong convexity of the squared loss.

To derive the learning rates, we need to consider the approximation ability of $\mathcal{H_K}$ with respect to its capacity and $f_{\rho}$ in $L_{\rho_X}^{2}$.
Since the original regularizer $\langle f,f\rangle_{\mathcal{H_K}}$ in RKKS fails to quantify complexity of a hypothesis, here we use the empirical regularizer $\langle f, T f \rangle_{\mathcal{H_K}}$ in Definition~\ref{defemkernel} as an alternative. Note that, other RKHS regularizers, such as $\langle f,f \rangle_{{\mathcal{H_{\bar K}}}}$ in Definition~\ref{defassrkks}, is also acceptable, but the used $\langle f, Tf \rangle_{{\mathcal{H_K}}}$ will result in elegant and concise theoretical results.
Accordingly, the approximation ability of $\mathcal{H_K}$ can be characterised by the regularization error.
\begin{assumption}\label{assreg} (regularity condition)
	The regularization error of $\mathcal{H_K}$ is defined as
	\begin{equation}\label{Dlamdadef}
	D(\lambda)=\inf_{f \in \mathcal{H_K}} \Big\{ \mathcal{E}(f) - \mathcal{E}(f_{\rho}) + \lambda \langle f, Tf \rangle_{{\mathcal{H_K}}} \Big\} \,.
	\end{equation}
	The target function $f_{\rho}$ can be approximated by $\mathcal{H_K}$ with exponent $0 < \beta \leq 1$ if there exists a constant $C_0$ such that
	\begin{equation}\label{Dlambda}
	D(\lambda) \leq C_0\lambda^{\beta},~~\forall \lambda>0\,.
	\end{equation}
\end{assumption}
\noindent{\bf Remark:} This is a natural assumption and approximation theory requires it, e.g., \cite{Wu2006Learning,Wang2011Optimal,Steinwart2008SVM}.
Note that $\beta=1$ is the best choice as we expect, which is equivalent to $f_{\rho} \in \mathcal{H_K}$ when $\mathcal{H_K}$ is dense.

Furthermore, to quantitatively understand how the complexity of $\mathcal{H_K}$ affects the learning ability of algorithm~\eqref{fzrs}, we need the capacity (roughly speaking the ``size'') of $\mathcal{H_K}$ measured by covering numbers.
\begin{definition} (covering numbers \cite{cucker2007learning,shi2019sparse})
	For a subset $Q$ of $C(X)$ and $\epsilon> 0$, the \emph{covering number} $\mathscr{N}(Q, \epsilon)$ is the minimal
	integer $l \in \mathbb{N}$ such that there exist $l$ disks with radius $\epsilon$ covering $Q$.
\end{definition}
In this paper, the covering numbers of balls are defined by
\begin{equation}\label{BRradius}
\mathcal{B}_R = \{ f \in \mathcal{H_K}: \sqrt{\langle f, Tf \rangle_{{\mathcal{H_K}}}} \leq R \}\,,
\end{equation}
as subsets of $L^{\infty}(X)$.
{Note that the used $R$ in Eq.~\eqref{BRradius} and $r$ in problem~\eqref{fzrs} admits $R=Cr$ for some positive constant $C$, as the definition of such non-negative inner product leads to a hyper-sphere with different radius.
	Hence there is no difference of using $R$ or $r$ in our analysis and thus we directly use $R$ for convenience.}

\begin{assumption}\label{asscap} (capacity)
	We assume that for some $s>0$ and $C_s>0$ such that
	\begin{equation}\label{assumpN}
	\log \mathscr{N}(\mathcal{B}_1,\epsilon) \leq C_s \Big(\frac{1}{\epsilon}\Big)^s, \quad \forall \epsilon>0\,.
	\end{equation}
\end{assumption}
\noindent{\bf Remark:} This is a standard assumption to measure the capacity of $\mathcal{H_K}$ that follows with that of RKHS \cite{cucker2007learning,Wang2011Optimal,shi2019sparse}, 
When $X$ is bounded in $\mathbb{R}^d$ and $k \in C^{\tau}(X \times X)$, Eq.~\eqref{assumpN} always holds true with $s=\frac{2d}{\tau}$.
In particular, if $k \in C^{\infty}(X \times X)$, Eq.~\eqref{assumpN} is still valid for an arbitrary small $s>0$.

{It can be noticed that, the capacity of a RKHS can be also measured by eigenvalue decay of the PSD kernel matrix, which has been has been fully studied, e.g., \cite{bach2013sharp}. A small RKHS indicates a fast eigenvalue decay so as to obtain a promising prediction performance. In other words, functions in the RKHS are potentially smoother than what is necessary, which means an arbitrary small $s$ in Assumption~\ref{assumpN}.} Nevertheless, eigenvalue decay of the indefinite kernel matrix has not been studied before due to the extra negative eigenvalues.  
By virtue of eigenvalue decomposition $\bm K = \bm K_+ - \bm K_-$ with two PSD matrices $\bm K_{\pm}$, we can easily make the assumption for $\bm K $ based on the eigenvalue decay of $\bm K_{\pm}$.
\begin{assumption}\label{eigenassump}
	(eigenvalue assumption)
	Suppose that the indefinite kernel matrix $\bm K = \bm V \bm \Sigma \bm V^{\!\top}$ has $p$ positive eigenvalues, $q$ negative eigenvalues, and $m-p-q$ zero eigenvalues, i.e., $\bm \Sigma = \bm \Sigma_+ + \bm \Sigma_-$, where $\bm \Sigma_+ = \diag( \sigma_1, \sigma_2, \cdots \sigma_{p}, 0,\cdots, 0)$,
	$\bm \Sigma_- = \diag(0,\cdots,0, \sigma_{m-q+1}, \cdots, \sigma_{m})$ with the decreasing order $\sigma_1 \geq \cdots \geq \sigma_p > 0 > \sigma_{m-q+1} \geq \cdots \geq \sigma_{m}$ and $\sigma_{p+1} = \sigma_{p+2}  = \cdots = \sigma_{m-q} = 0$.
	Here we assume that its (positive) largest eigenvalue satisfies $\sigma_1 \geq c_1 m^{\eta_1}$ with $c_1>0$, $\eta_1 > 0$ and its smallest (negative) eigenvalue admits $\sigma_m \leq c_m m^{\eta_2}$ with $c_m<0$, $\eta_2 > 0$.
	And we denote $\eta \coloneqq  \min\{ \eta_1, \eta_2 \} $.
\end{assumption}
\noindent{\bf Remark:} {Our assumption only requires the lower bound of the largest (positive) eigenvalue and the upper bound of the smallest (negative) eigenvalue, which is weaker than the common decay of a PSD kernel matrix, e.g., polynomial/exponential decay.
	In particular, if we take these common eigenvalue decays of $\bm K_{\pm}$, then our assumption on $\sigma_1$ and $\sigma_m$ is naturally satisfied.}
To be specific, Bach \cite{bach2013sharp} considers three eigenvalue decays of a PSD kernel matrix, including i) the exponential decay $\sigma_i \propto m e^{-ci}$ with $c > 0$, ii) the polynomial decay $\sigma_i \propto m i^{-2t}$ with $t \geq 1$, and iii) the slowest decay with $\sigma_i \propto m /i$.
Hence, under such three eigenvalue decays of $\bm K_{\pm}$, then our assumption on $\sigma_1 \geq c_1 m^{\eta_1}$ and  $\sigma_m \leq c_m m^{\eta_2}$ always holds.
Specifically, although the number of positive/negative eigenvalues depends on the sample set, our theoretical results will be independent of the unknown $p$ and $q$.

Formally, our main result about least squares regularized regression in RKKS is stated as follows.
\begin{theorem}\label{theorem1ls}
	Suppose that $|f_{\rho}(x)|\leq M^*$ with $M^* \geq 1$ in Assumption~\ref{assrho}, $\rho$ satisfies the regularity condition in Eq.~\eqref{Dlambda} with $0 < \beta \leq 1$ in Assumption~\ref{assreg}, the indefinite kernel matrix $\bm K$ satisfies the eigenvalue assumption in Assumption~\ref{eigenassump} with $\eta = \min \{ \eta_1, \eta_2 \} > 0$.
	For some $s>0$ in Assumption~\ref{asscap}, take $\lambda\coloneqq m^{-\gamma}$ with $ 0 < \gamma \leq 1$.
	Let
	\begin{equation}\label{epst}
	0 < \epsilon < \frac{1}{s} - (\gamma+ s\gamma-1)(2+s)\,.
	\end{equation}
	Then for $0<\delta<1$ with confidence $1-\delta$,  when $\gamma + \eta > 1$,
	we have
	\begin{equation*}\label{result1ls}
	\begin{split}
	\big\| \pi_{M^*}(f_{\bm{z},\lambda})  - f_{\rho} \big\|^2_{L_{\rho_X}^{2}}  \leq  \widetilde{C}  \left(\log\frac{2}{\epsilon} \right)^{\!2} \log\frac{2}{\delta} m^{- \Theta } \,,
	\end{split}
	\end{equation*}
	where $\widetilde{C}$ is a constant independent of $m$ or $\delta$ and the power index $\Theta$ is
	\begin{equation}\label{rate1ls}
	\begin{split}
	\Theta & \!=\! \min \bigg\{ \gamma \beta, \gamma + \eta -1, \frac{2 - s \gamma (1 -\beta) }{2(1+s)}, \frac{2-s(1-\eta)}{2(1+s)}, \frac{1-s(\gamma+s\gamma-1)(2+s)-s\epsilon}{1+s} \bigg\}\,,
	\end{split}
	\end{equation}
	where $\eta$ is further restricted by $\max\{0, 1-2/s\} < \eta < 1$ for a positive $\Theta$, i.e., a valid learning rate.
\end{theorem}

We hence directly have the following corollary that corresponds to learning rates in RKHS.

\begin{corollary}\label{corrkhs} (link to learning rates in RKHS)
	When $\eta \geq 1$, the power index $\Theta$ in Eq.~\eqref{rate1ls} can be simplified as
	\begin{equation}\label{rate1rkhs}
	\begin{split}
	\Theta & = \min \bigg\{ \gamma \beta, \frac{2 - s \gamma (1 - \beta) }{2(1+ s)}, \frac{1-(s\gamma(1+s)-s)(2+s)-s\epsilon}{1+s} \bigg\},
	\end{split}
	\end{equation}
	which is actually the learning rate for least squares regularized regression in RKHS, independent of $\eta$.
\end{corollary}
\noindent{\bf Remark:} We provide learning rates in RKKS in Theorem~\ref{theorem1ls} and also demonstrate the relation of the derived learning rates between RKKS and RKHS in Corollary~\ref{corrkhs}.
We make the following remarks.\\
\textit{i})
In Theorem~\ref{theorem1ls}, our results choose $\lambda \coloneqq  m^{-\gamma}$ and the radius $R$ (or $r$) is implicit in Eq.~\eqref{rate1ls}.
The estimation for $R$ depends on a bound for $\lambda \langle f_{\bm{z},\lambda} ,Tf_{\bm{z},\lambda} \rangle_{\mathcal{H_K}}$, see Lemma~\ref{lemmafzr} for details.
Note that $s$ can be arbitrarily small when the kernel $k$ is $C^{\infty}(X \times X)$.
In this case, $\Theta$ in Eq.~\eqref{rate1ls} can be arbitrarily close to $\min ({\gamma \beta},  \gamma + \eta -1)$. \\
\textit{ii}) Corollary~\ref{corrkhs} derives the learning rates in RKHS, which recovers the result of \cite{Wang2011Optimal} for least squares in RKHS.
That is, when choosing $\beta=1$ and $s$ is small enough, the derived learning rate in Corollary~\ref{corrkhs} can be arbitrary close to 1, and hence is optimal \cite{Wang2011Optimal}.\\
\textit{iii})
Based on Theorem~\ref{theorem1ls} and Corollary~\ref{corrkhs}, we find that if $\eta=\min \{  \eta_1, \eta_2 \} \geq 1$, our analysis for RKKS is the same as that in RKHS.
This is the best case.
However, if $\eta=\min \{  \eta_1, \eta_2 \} \leq 1$, the derived learning rate in RKKS demonstrated by Eq.~\eqref{rate1ls} is not faster than that in RKHS.
It is reasonable since the spanning space of RKKS is larger than that of RKHS.

The proof of Theorem~\ref{theorem1ls} is fairly technical and lengthy, and we briefly sketch some main ideas in the next section.

Furthermore, if problem~\eqref{fzrs} considers some nonnegative regularizers, such as $\langle f,Tf\rangle_{\mathcal{H_K}}$ in problem~\eqref{lsprimalreg}, or $\| f\|_{{\mathcal{H_{\bar K}}}}^2$ in Definition~\ref{defassrkks}, the analysis would be simplified due to the used nonnegative regularizer.
To be specific, denote
$\overline{f_{\bm{z},\lambda}} \coloneqq  \argmin_{f \in\mathcal{B} (r)} \left\{ \frac{1}{m} \sum_{i=1}^{m} \big(f(\bm x_i) - y_i \big)^2 + \lambda \| f\|_{{\mathcal{H_{\bar K}}}}^2 \right\}$ as demonstrated by \cite{oglic18a}, its learning rate could be given by the following corollary.
\begin{corollary}\label{corolr}
	Under the same assumption with Theorem~\ref{theorem1ls} (without the eigenvalue assumption), by defining the regularization error as
	\begin{equation*}
	D'(\lambda)=\inf_{f \in \mathcal{H_K}} \Big\{ \mathcal{E}(f) - \mathcal{E}(f_{\rho}) + \lambda \| f\|_{{\mathcal{H_{\bar K}}}}^2 \Big\}\,,
	\end{equation*}
	satisfying $  D'(\lambda) \leq C'_0\lambda^{\beta'}$ with a constant $C'_0$ and $\beta' \in (0,1]$, we have
	\begin{equation*}
	\big\| \pi_{M^*}(\overline{f_{\bm{z},\lambda}})  - f_{\rho} \big\|^2_{L_{\rho_X}^{2}}  \leq  \widetilde{C}'  \left(\log\frac{2}{\epsilon} \right)^{\!2} \log\frac{2}{\delta} m^{- \Theta' }\,,
	\end{equation*}
	where $\widetilde{C}'$ is a constant independent of $m$ or $\delta$ and the power index $\Theta'$ is defined as Eq.~\eqref{rate1rkhs} with $\beta'$.
\end{corollary}
\noindent{\bf Remark:} Corollary~\ref{corolr} can be in fact applied to the model in \cite{oglic18a}. 
Note that the learning rates would be effected by different regularizers, as indicated by the regularization error in Assumption~\ref{assreg}.
In Table~\ref{Tabcomn} we summarize the learning rates of problem~\eqref{fzrs} with different non-negative regularizers.
Although the associated Hilbert space norms generated by different decomposition of the Krein space are topologically equivalent \cite{langer1962spektraltheoriej}, the derived learning rates cannot be ensured to be the same due to their respective spanning/solving spaces.
Besides, Oglic and G\"artner \cite{oglic2019scalable} demonstrate that, stabilization of support vector machine (SVM) in RKKS can be transformed to a risk minimization problem with a PSD kernel matrix by taking the absolute value of negative eigenvalues of the original indefinite one. 
That means, stabilization of SVM in RKKS could also achieve the same convergence behavior as risk minimization with a PSD kernel matrix in RKHS, e.g., \cite{steinwart2007fast}.
{Accordingly, we conclude that, risk minimization in RKKS with the original regularizer induced by the indefinite inner product is general, provide the worst case, and can be improved to the same learning rates as other settings, e.g., minimization in RKKS but with non-negative regularizers, least-squares in RKHS.} 

\begin{table}[t]
	\centering
	\begin{threeparttable}
		\caption{Comparisons of different least squares regression problems.}
		\label{Tabcomn}
		\begin{tabular}{cccccccc}
			\toprule
			learning problem in RKKS  &learning rates  \cr
			\midrule
			$ f_{\bm{z},\lambda} \coloneqq  \argmin \limits_{f \in\mathcal{B} (r)}  \left\{ \mathcal{E}_{\bm z}(f)+ \lambda \langle f,f\rangle_{\mathcal{H_K}} \right\}$  &Eq.~\eqref{rate1ls} \\
			$ \overline{f_{\bm{z},\lambda}} \!\coloneqq  \argmin \limits_{f \in\mathcal{B} (r)} \left\{ \mathcal{E}_{\bm z}(f)+\! \lambda \langle f,f\rangle_{{\mathcal{H_{\bar K}}}} \right\}$ & Corollary~\ref{corolr} (applied to \cite{oglic18a}) \\ $ \widetilde{f_{\bm{z},\lambda}} \!\coloneqq \! \argmin \limits_{f \in\mathcal{B} (r)} \left\{ \mathcal{E}_{\bm z}(f)+\! \lambda \langle f,Tf\rangle_{\mathcal{H_K}} \!\right\}$   &Corollary~\ref{corolr} ($\beta$ is different)  \\
			\bottomrule
		\end{tabular}
	\end{threeparttable}
\end{table}

\section{Solution to Regularized Least-Squares in RKKS}
\label{sec:algo}

In this section, we study the optimization problem~\eqref{fzrs}, obtain a globally optimal solution, and provide another regularization scheme to aid our analysis.

Problem \eqref{fzrs} can be formulated as
\begin{equation}\label{fzrsdual}
\bm \alpha_{\bm{z},\lambda}\! \coloneqq \! \argmin_{ \bm \alpha \in \mathbb{R}^m: \bm \alpha^{\!\top} \bm K^2 \bm \alpha \leq mr^2} \bigg\{ \frac{1}{m} \| \bm K \bm \alpha - \bm y \|_2^2 +  \lambda \bm \alpha^{\!\top} \bm K \bm \alpha \bigg\}\,,
\end{equation}
where the output is $\bm y=[y_1, y_2, \cdots, y_m]^{\!\top}$.
We can see that the above regularized risk minimization problem is in essence non-convex due to the non-positive definiteness of $\bm K$. But more exactly, problem~\eqref{fzrsdual} is non-convex when $\frac{1}{m} \bm K^2 + \lambda \bm K$ is indefinite. 
This condition always holds in practice due to $m \gg \lambda$.
Following \cite{oglic18a}, we do not strictly distinguish between the two differences in this paper.
This is because, approximation analysis considers the $m \rightarrow \infty$ and $\lambda \rightarrow 0$ case, so it always holds true when $m$ is large enough.
Even if $\frac{1}{m} \bm K^2 + \lambda \bm K$ is PSD, our analysis for problem~\eqref{fzrsdual} is still applicable and reduces to a special case (i.e., using a RKHS regularizer), of which the learning rates are demonstrated by Corollary~\ref{corolr}.


To obtain a global minimum of problem \eqref{fzrsdual}, we need the following proposition.
\begin{proposition}\label{proineqeq}
	Problem~\eqref{fzrsdual} is equivalent to
	\begin{equation}\label{fzrsdualnew}
	\bm \alpha_{\bm{z},\lambda}\! \coloneqq \! \argmin_{\bm \alpha \in \mathbb{R}^m: \bm \alpha^{\!\top} \bm K^2 \bm \alpha = mr^2} \bigg\{ \frac{1}{m} \| \bm K \bm \alpha - \bm y \|_2^2 +  \lambda \bm \alpha^{\!\top} \bm K \bm \alpha \bigg\}.
	\end{equation}
	\vspace{-0.1cm}
\end{proposition}

\begin{proof}
	Denote the objective function in problem~\eqref{fzrsdual} as $F(\bm \alpha) = \frac{1}{m} \| \bm K \bm \alpha - \bm y \|_2^2 +  \lambda \bm \alpha^{\!\top} \bm K \bm \alpha$, 
	we aim to prove that the solution $\bm \alpha^* \coloneqq \argmin_{\bm \alpha} F(\bm \alpha)$ of this unconstrained optimization problem would be unbounded.
	Due to the non-positive definiteness of $\frac{1}{m} \bm K^2 + \lambda \bm K$, there exists an initial solution $\bm \alpha_0$ such that
	\begin{equation*}
	\bm \alpha_0^{\!\top} \Big( \frac{1}{m} \bm K^2 + \lambda \bm K \Big) \bm \alpha_0 < 0\,.
	\end{equation*}
	By constructing a solving sequence $\{ \bm \alpha_i \}_{i=0}^{\infty}$ admitting $ \bm \alpha_{i+1} \coloneqq  c \bm \alpha_{i}$ with $c > 1$, we have
	\begin{equation*}
	{F}(c \bm \alpha_{i+1}) - c {F} (\bm \alpha_i) = c(c-1) \bm \alpha_i^{\!\top} \Big( \frac{1}{m} \bm K^2 + \lambda \bm K \Big) \bm \alpha_i - \frac{c-1}{m} \| \bm y \|_2^2 < 0 \,,
	\end{equation*}
	which indicates that, after the $t$-th iteration, ${F}(\bm \alpha_t) < c^t {F} (\bm \alpha_0) < 0$ and $\| \bm \alpha_t\|_2 = c^t \| \bm \alpha_0 \|_2 $ with $c>1$.
	Therefore, the minimum $F(\bm \alpha^*)$ is unbounded, and tends to negative infinity.
	In this case, $\| \bm \alpha^* \|_2$ would also approach to infinity, i.e., a meaningless solution.
	Based on the above analyses, for problem $\min_{\bm \alpha} F(\bm \alpha)$, by introducing the constraint $\bm \alpha^{\!\top} \bm K^2 \bm \alpha \leq mr^2$, its solution is obtained on the hyper-sphere, i.e., $\bm \alpha^{\!\top} \bm K^2 \bm \alpha = mr^2$, which concludes the proof.
\end{proof}

As demonstrated by Proposition~\ref{proineqeq}, the inequality constraint in problem~\eqref{fzrsdual} can be transformed into an equality constraint, which is also suitable to problem~\eqref{fzrs}.
{Then, albeit non-convex, problem~\eqref{fzrsdualnew} can be formulated as solving a constrained eigenvalue problem \cite{Gander1988A,oglic18a}, yielding an optimal solution with closed-form.}\footnote{As a generalized trust-region subproblem, problem~\eqref{fzrsdualnew} can be also solved by the S-lemma with equality to yield a globally optimal solution \cite{Adachi2017Eigenvalue,Xia2016S}.}
Accordingly, the optimal solution $\bm \alpha_{\bm z, \lambda}$ is given by
\begin{equation}\label{fzlambdasol}
\bm \alpha_{\bm z, \lambda} = \frac{1}{m}(\lambda \bm I - {\mu} \bm K)^{\dag} \bm y\,,
\end{equation}
where the notation $(\cdot)^{\dag}$ denotes the pseudo-inverse, $\bm I$ is the identity matrix,  and ${\mu}$ is the smallest real eigenvalue of the following non-Hermitian matrix
\begin{equation}\label{Incre}
{\bm G} = \left[                
\begin{array}{ccc} 
\lambda \bm K^{\dag} & -\bm I\\
-{\bm y \bm y^{\!\top}}/{m^3 r^2} & \lambda \bm K^{\dag} \\
\end{array}
\right]\,,
\end{equation}
where $\bm K^{\dag}$ is the pseudo-inverse of $\bm K$, i.e. $\bm K^{\dag} = \bm V \diag \big( \bm \Sigma_1, \bm 0_{m-p-q}, \bm \Sigma_2 \big) \bm V^{\!\top}$
with two invertible diagonal matrices
\begin{equation}\label{diagsimg}
\bm \Sigma_1 = \diag \Big( \frac{\lambda}{\sigma_1}, \dots, \frac{\lambda}{\sigma_p} \Big),~~ \bm \Sigma_2 = \diag \Big( \frac{\lambda}{\sigma_{m-q+1}}, \dots, \frac{\lambda}{\sigma_{m}} \Big).
\end{equation}
It is clear that we cannot directly calculate $\mu$.
However, $\mu$ is very important in our analysis and thus we attempt to estimate it based on matrix perturbation theory \cite{stewart1990matrix}.
We will detail this in Section~\ref{sec:prooframe}.

Besides, to aid our analysis, we introduce another nonnegative regularization scheme in RKKS to problem~\eqref{fzrs}
\begin{equation}\label{lsprimalreg}
\begin{split}
\widetilde{f_{\bm z, \lambda}} \!\coloneqq \mathop{\mathrm{argmin}}\limits_{f \in \mathcal{B}(r)}\left\{ \frac{1}{m} \sum_{i=1}^{m} \big( f(\bm x_i) - y_i \big)^2 \!+\! \lambda \langle f, Tf \rangle_{{\mathcal{H_K}}}  \right\},  \\
\end{split}
\end{equation}
where the empirical covariance operator $T$ is defined in RKKS but nonnegative, see Definition~\ref{defemkernel}.
Based on the above regularized risk minimization problem, following Theorem~2 in \cite{oglic18a}, we can easily prove the following representer theorem with omitting the proof here.
\begin{theorem}\label{theorepre}
	Let $\widetilde{f_{\bm z, \lambda}}$ be an optimal solution to the regularized risk minimization problem~\eqref{lsprimalreg}, then it admits the expansion $\widetilde{f_{\bm z, \lambda}} = \sum_{i=1}^m \alpha_i k(\bm x_i, \cdot)$ by the reproducing kernel $k$ with $\alpha_i \in \mathbb{R}$.
\end{theorem}

By virtue of Theorem~\ref{theorepre} and Eq.~\eqref{emkernel}, the regularizer can be represented as
\begin{equation*}
\begin{split}
\langle f, Tf \rangle_{{\mathcal{H_K}}} &= \frac{1}{m} \sum_{i,i'=1}^{m} \alpha_i \alpha_{i'} \sum_{j=1}^{m} k(\bm x_i, \bm x_j) k(\bm x_{i'}, \bm x_j) = \frac{1}{m} \bm \alpha^{\!\top} \bm K^2 \bm \alpha.
\end{split}
\end{equation*}
Accordingly, problem~\eqref{lsprimalreg} can be formulated as
\begin{equation}\label{fzrsreg}
\widetilde{\bm \alpha_{\bm z, \lambda}} \coloneqq \mathop{\mathrm{argmin}}\limits_{\bm \alpha \in \mathbb{R}^m: \bm \alpha^{\!\top} \bm K^2 \bm \alpha = mr^2} \left\{ \frac{1}{m} \| \bm K \bm \alpha - \bm y \|_2^2 + \frac{\lambda}{m} \bm \alpha^{\!\top} \bm K^2 \bm \alpha \right\}\,,
\end{equation}
with $\widetilde{\bm \alpha_{\bm z, \lambda}} = -\frac{1}{m \widetilde{\mu}} \bm K^{\dag} \bm y$, and $\widetilde{\mu}$ is the smallest real eigenvalue of the matrix
\begin{equation}\label{k2al}
\widetilde{\bm G} = \left[               
\begin{array}{ccc}  
\bm 0_{m} & -\bm I\\
-\bm y \bm y^{\!\top}/m^3 r^2 & \bm 0_{m} \\
\end{array}
\right]\,.
\end{equation}
By Sylvester's determinant identity, we directly calculate the largest and smallest real eigenvalues of $\widetilde{\bm G}$ as $\frac{\| \bm y \|_2}{\sqrt{m}mr}$ and $-\frac{\| \bm y \|_2}{\sqrt{m}mr}$, respectively.
So we have $\widetilde{\mu} = -\frac{\| \bm y \|_2}{m\sqrt{m}r} < 0$.
Note that the regularizer in problem~\eqref{lsprimalreg} can be also chosen to be other RKHS regularizers, such as $\langle f,f \rangle_{{\mathcal{H_{\bar K}}}}$ in Definition~\ref{defassrkks}.
But using the empirical kernel regularizer $\langle f, Tf \rangle_{{\mathcal{H_K}}}$, one obtains elegant and concise theoretical results, i.e., directly compute $\widetilde{\mu}$.

\section{Framework of proofs}
\label{sec:prooframe}
\vspace{-0.1cm}
In this section, we establish the framework of proofs for Theorem~\ref{theorem1ls}.
By the modified error decomposition technique in section~\ref{sec:ed}, the total error can be decomposed into the regularization
error, the sample error, and an additional hypothesis error.
We detail the estimates for the hypothesis error in section \ref{sec:hypoerr}.
These two points are the main elements on novelty in the proof.
We briefly introduce estimates for the sample error in section~\ref{sec:se} and derive the learning rates in section~\ref{sec:learn}.

\subsection{Error Decomposition}
\label{sec:ed}
In order to estimate error $\| \pi_{M^*}(f_{\bm{z},\lambda})  - f_{\rho} \|$ in the $L_{\rho_X}^{2}$ space, i.e., to bound $\| \pi_{B}(f_{\bm{z},\lambda})  - f_{\rho} \|$ for any $B \geq M^*$, we need to estimate the excess error $\mathcal{E}\big(\pi_B(f_{\bm{z},\lambda})\big) - \mathcal{E}(f_{\rho})$ which can be conducted by an error decomposition technique \cite{cucker2007learning}.
However, since $\langle f_{\bm{z},\lambda},f_{\bm{z},\lambda} \rangle_{\mathcal{H_K}}$ might be negative, traditional techniques are invalid.
Formally, our modified error decomposition technique is given by the following proposition by introducing an additional hypothesis error.
\begin{proposition}\label{errdec}
	Let
	$f_\lambda = \argmin_{f \in \mathcal{H_K}} \Big\{ \mathcal{E}(f) - \mathcal{E}(f_{\rho}) + \lambda \langle f,Tf\rangle_{\mathcal{H_K}}  \Big\}$, 
	then $\mathcal{E}\big(\pi_B(f_{\bm{z},\lambda})\big) - \mathcal{E}(f_{\rho})$ can be bounded by
	\begin{equation*}
	\begin{split}
	\mathcal{E}\!\big(\pi_B(f_{\bm{z},\lambda})\big) - \mathcal{E}(f_{\rho})
	&  \leq  \mathcal{E} \big(\pi_B(f_{\bm{z},\lambda})\big) - \mathcal{E}(f_{\rho}) + \lambda \langle f_{\bm{z},\lambda} , T f_{\bm{z},\lambda} \rangle_{\!\mathcal{H_K}} \\
	& \leq D(\lambda) + S(\bm z, \lambda) + P(\bm z, \lambda) \,,
	\end{split}
	\end{equation*}
	where $D(\lambda)$ is the regularization error defined by Eq.~\eqref{Dlamdadef}.
	The sample error $S(\bm z, \lambda)$ is given by
	\begin{equation*}
	S(\bm z, \lambda)  = \mathcal{E}\big(\pi_B(f_{\bm{z},\lambda})\big) -  \mathcal{E}_{\bm{z}}\big(\pi_B(f_{\bm{z},\lambda})\big) + \mathcal{E}_{\bm{z}}\big(f_{\lambda}\big) - \mathcal{E}\big(f_{\lambda}\big)\,.
	\end{equation*}
	The introduced hypothesis error $P(\bm z, \lambda)$ is defined by
	\begin{equation}\label{hypoformu}
	\begin{split}
	P(\bm z, \lambda) &=\! {\mathcal{E}_{\bm{z}}\big(f_{\bm{z},\lambda}\big)} + {\lambda \langle f_{\bm{z},\lambda} ,Tf_{\bm{z},\lambda} \rangle_{\mathcal{H_K}} } 
	- {\mathcal{E}_{\bm{z}}\big(\widetilde{f_{\bm{z},\lambda}}\big)} - {\lambda \langle \widetilde{f_{\bm{z},\lambda}}, T\widetilde{f_{\bm{z},\lambda}} \rangle_{\mathcal{H_K}} } \,,
	\end{split}
	\end{equation}
	where $f_{\bm z, \lambda}$ and $\widetilde{f_{\bm z, \lambda}}$ are optimal solutions of problem~\eqref{fzrs} and problem~\eqref{lsprimalreg}, respectively.
\end{proposition}
\begin{proof}
	We write $\mathcal{E}\big(\pi_B(f_{\bm{z},\lambda})\big) - \mathcal{E}(f_{\rho}) +\lambda \langle f_{\bm{z},\lambda} ,Tf_{\bm{z},\lambda} \rangle_{\mathcal{H_K}}$ as
	\begin{equation*}
	\begin{split}
	& \mathcal{E}\big(\pi_B(f_{\bm{z},\lambda})\big) - \mathcal{E}(f_{\rho}) +\lambda \langle f_{\bm{z},\lambda} ,Tf_{\bm{z},\lambda} \rangle_{\mathcal{H_K}} 
	=  \Big\{ \mathcal{E}\big(\pi_B(f_{\bm{z},\lambda})\big) - \mathcal{E}_{\bm{z}}\big(\pi_B(f_{\bm{z},\lambda})\big) \Big\} \\
	& \quad + \Big\{ \mathcal{E}_{\bm{z}}\big(\pi_B(f_{\bm{z},\lambda})\big) +{\lambda \langle f_{\bm{z},\lambda} ,Tf_{\bm{z},\lambda} \rangle_{\mathcal{H_K}} }  \Big\}  - \Big\{ \mathcal{E}_{\bm{z}}(f_{\lambda}) + \lambda  \langle f_{\lambda} ,Tf_{\lambda} \rangle_{\mathcal{H_K}}  \Big\}  + \Big\{ \mathcal{E}_{\bm{z}}(f_{\lambda})  - \mathcal{E}(f_{\lambda}) \Big\} \\
	& \quad +  \Big\{ \mathcal{E}(f_{\lambda}) -  \mathcal{E}(f_{\rho}) + \lambda \langle f_{\lambda} ,Tf_{\lambda} \rangle_{\mathcal{H_K}} \Big\} \\
	& \leq  D(\lambda) + P(\bm z, \lambda)  + S(\bm z, \lambda)  \,,
	\end{split}
	\end{equation*}
	where we use $\mathcal{E}_{\bm z}\big(\pi_B(f_{\bm{z},\lambda})\big)  \leq \mathcal{E}_{\bm z}\big(f_{\bm{z},\lambda}\big)$ in the first inequality, and the second inequality holds by the condition that $\widetilde{f_{\bm{z},\lambda}}$ is a global minimizer of problem~\eqref{lsprimalreg}.
\end{proof}
It can be found that the additional hypothesis error stems from the difference between $\langle f_{\bm{z},\lambda}, f_{\bm{z},\lambda}\rangle_{\mathcal{H_{K}}}$-regularization and $\langle f_{\bm{z},\lambda} ,Tf_{\bm{z},\lambda} \rangle_{\mathcal{H_K}}$-regularization in essence.
Hence, we estimate the introduced hypothesis error in the following descriptions.

\subsection{Bound Hypothesis Error}
\label{sec:hypoerr}

Since $\widetilde{f_{\bm z, \lambda}}$ is an optimal solution of problem~\eqref{lsprimalreg}, obviously, we have $P(\bm z, \lambda) \geq 0$.
To bound the hypothesis error, we need to estimate the objective function value difference of the two learning problems~\eqref{fzrs} and \eqref{fzrsreg} by the following proposition.
\begin{proposition}\label{propos2pz}
	Suppose that the spectrum of the indefinite kernel matrix $\bm K$ satisfies Assumption~\ref{eigenassump}, denote the condition number of two invertible matrices $\bm \Sigma_1$,  $\bm \Sigma_2$ in Eq.~\eqref{diagsimg} as $C_1, C_2 < \infty$.
	When $\eta + \gamma > 1$ with $\eta = \min \{ \eta_1, \eta_2 \}$, the hypothesis error defined in Eq.~\eqref{hypoformu} holds with probability 1 such that
	\begin{equation*}
	P(\bm z, \lambda) \leq \widetilde{C_1} m^{-\Theta_1} \,,
	\end{equation*}
	where $\widetilde{C_1} \coloneqq  2Mr + 2M^2 \big( \frac{-c_m}{C_2} + \frac{M^2}{r^2} + \frac{C_{1}}{c_1} \big) $ and the power index is
	$ \Theta_1 = \min \big\{ 1, \gamma+\eta-1 \big\}$.
\end{proposition}
\begin{proof}
	The proof can be found in Section~\ref{ratehypo}.
\end{proof}
{\bf Remark:} The condition number of invertible matrices is finite, which is mild as demonstrated by \cite{gao2015minimax}.

{ In the next, we give the proof of Proposition~\ref{propos2pz}. For better presentation, we divide the proof into three parts: in Section~\ref{decomhypo}, we decompose the hypothesis error $P(\bm z, \lambda)$ into the sum of two terms that would depend on ${\mu}$, i.e., the smallest real eigenvalue of a non-Hermitian matrix $\bm G$ in Eq.~\eqref{Incre}.
	Then we estimate ${\mu}$ in Section~\ref{estimu} so as to bound $P_2(\bm z, \lambda)$ and $P(\bm z,\lambda)$ in Section~\ref{ratehypo}.}

\subsubsection{Decomposition of hypothesis error}
\label{decomhypo}
{The hypothesis error $P(\bm z, \lambda)$ can be decomposed into the sum of two parts that depend on $\mu$ and $\tilde{\mu}$ by the following proposition.}

\begin{proposition}\label{prohyto}
	{Given the hypothesis error $P(\bm z, \lambda)$ defined in Eq.~\eqref{hypoformu}, it can be decomposed as
		\begin{equation*}
		\begin{split}
		P(\bm z, \lambda) &= P_1(\bm z, \lambda) + P_2(\bm z, \lambda) := -\frac{2}{m^2 \tilde{\mu}} \bm y^{\!\top} \bm K \bm K^{\dag} \bm y -\frac{2}{m^2} \bm y^{\!\top} \bm K (\lambda \bm I - {\mu} \bm K)^{\dag}  \bm y\,,
		\end{split}
		\end{equation*}
		where $P_1(\bm z, \lambda) $ depends on $\widetilde{\mu} := -\frac{\| \bm y \|_2}{m\sqrt{m}r}$ and $P_2(\bm z, \lambda) $ depends on ${\mu}$, i.e., the smallest real eigenvalue of a non-Hermitian matrix $\bm G$ in Eq.~\eqref{Incre}.}
\end{proposition}
\begin{proof}
	According to the definition of the hypothesis error $P(\bm z, \lambda)$, we have
	\begin{equation*}
	\begin{split}
	P(\bm z, \lambda) &= {\mathcal{E}_{\bm{z}}\big(f_{\bm{z},\lambda}\big)} + {\lambda \langle f_{\bm{z},\lambda} ,Tf_{\bm{z},\lambda} \rangle_{\mathcal{H_K}} } - {\mathcal{E}_{\bm{z}}\big(\widetilde{f_{\bm{z},\lambda}}\big)}
	- {\lambda \langle \widetilde{f_{\bm{z},\lambda}} ,T\widetilde{f_{\bm{z},\lambda}} \rangle_{\mathcal{H_K}} } \,,
	\end{split}
	\end{equation*}
	where $f_{\bm z, \lambda}$ and $\widetilde{f_{\bm z, \lambda}}$ are optimal solutions of problem~\eqref{fzrs} and problem~\eqref{lsprimalreg}, respectively.
	Therefore, both of them can be obtained on the hyper-sphere. Besides, the regularizer is $\bm \alpha_{z, \lambda}^{\!\top} \bm K^2 \bm \alpha_{z, \lambda} = mr^2$ can be canceled out in $P(\bm z, \lambda)$.
	Based on this, $P(\bm z, \lambda)$ can be further represented as
	\begin{equation*}
	\begin{split}
	P(\bm z, \lambda)  &= \frac{1}{m} \sum_{i=1}^{m} \big( f_{\bm z, \lambda}(\bm x_i) - y_i \big)^2 + \lambda\langle f_{\bm{z},\lambda} ,Tf_{\bm{z},\lambda} \rangle_{\mathcal{H_K}} 
	- \frac{1}{m} \sum_{i=1}^{m} \big( \widetilde{f_{\bm z, \lambda}}(\bm x_i) - y_i \big)^2 - \lambda \langle \widetilde{f_{\bm{z},\lambda}} ,T\widetilde{f_{\bm{z},\lambda}} \rangle_{\mathcal{H_K}} \\
	&= \frac{1}{m} \| \bm K \bm \alpha_{\bm z, \lambda} - \bm y \|_2^2 - \frac{1}{m} \| \bm K \widetilde{ \bm \alpha_{\bm z, \lambda}} - \bm y \|_2^2  =  \frac{2}{m} \bm y^{\!\top}\bm K \widetilde{\bm \alpha_{\bm z,\lambda}}  -\frac{2}{m} \bm y^{\!\top} \bm K \bm \alpha_{\bm z,\lambda} \\
	& = \underbrace{-\frac{2}{m^2 \tilde{\mu}} \bm y^{\!\top} \bm K \bm K^{\dag} \bm y}_{\triangleq P_1(\bm z,\lambda) } \underbrace{-\frac{2}{m^2} \bm y^{\!\top} \bm K (\lambda \bm I - {\mu} \bm K)^{\dag}  \bm y}_{\triangleq P_2(\bm z,\lambda) }\,.
	\end{split}
	\end{equation*}
\end{proof}

\subsubsection{Estimate ${\mu}$}
\label{estimu}
To bound $P(\bm z, \lambda)$, we need to bound $P_1(\bm z, \lambda)$ and $P_2(\bm z, \lambda)$ respectively.
The estimation for $P_1(\bm z, \lambda)$ is simple (we will illustrate it in the next subsection). However, $P_2(\bm z, \lambda)$ involves with ${\mu}$, i.e., the smallest real eigenvalue of a non-Hermitian matrix $\bm G$, which makes our estimation for $P_2(\bm z, \lambda)$ quite intractable.
Based on this, here we attempt to present an estimation for $\mu$ based on matrix perturbation theory \cite{stewart1990matrix}.

Typically, there are three classical and well-known perturbation bounds for matrix eigenvalues, including the Bauer-Fike theorem and the Hoffman-Wielandt theorem for diagonalizable matrices \cite{hoffman2003variation}, and Weyl's theorem for Hermitian matrices \cite{stewart1990matrix}.
However, $\bm G$ is neither Hermitian nor diagonalizable. To aid our proof, we need the following lemma.
\begin{lemma}\label{henrici}
	(Henrici theorem \cite{Chu1986Generalization})
	Let $\bm A$ be an $m \times m$ matrix with Schur decomposition $\bm Q^{{H}} \bm A \bm Q = \bm D + \bm U$, where $\bm Q$ is unitary, $\bm D$ is a diagonal matrix and $\bm U$ is a strict upper triangular matrix, with $(\cdot)^{{H}}$ denoting the Hermitian transpose.
	For each eigenvalue $\tilde{\sigma}$ of $\bm A + \widetilde{\Delta}$, there exists an eigenvalue $\sigma(\bm A)$ of $\bm A$ such that
	\begin{equation*}
	| \tilde{\sigma} - \sigma(\bm A) | \leq \max(\varsigma, \sqrt[^b]{\varsigma} )\,,~ \mbox{where}~ \varsigma \coloneqq  \| \widetilde{\Delta} \|_2 \sum_{i=1}^{b-1} \| \bm U \|_2^i\,,
	\end{equation*}
	where $b \leq m$ is the smallest integer satisfying $\bm U^b=0$, i.e., the nilpotent index of $\bm U$.
\end{lemma}
Based on the above lemma, $\mu$ admits the following representation.
\begin{proposition}\label{promu}
	Under the assumption of Proposition~\ref{propos2pz}, as the smallest real eigenvalue of a non-Hermitian matrix $\bm G$ in Eq.~\eqref{Incre}, ${\mu}$ admits the following expression
	\begin{equation}\label{mu1for}
	\mu =  \widetilde{c_a} \widetilde{\mu} + \widetilde{c_b} \widetilde{\mu}^2
	+ \left[ \frac{C_{2}}{c_m} +  \widetilde{c_d} \bigg( \frac{C_{1}}{c_1} - \frac{C_{2}}{c_m} \bigg) \right] m^{-(\gamma + \eta)} \,,
	\end{equation}
	with $\widetilde{c_a} \in [-1, 0) \bigcup (0,1]$, $\widetilde{c_b} \in [-1,1]$, $\widetilde{c_d} \in [0,1]$, and $\widetilde{\mu} := -\frac{\| \bm y \|_2}{m\sqrt{m}r} < 0$.
\end{proposition}
\begin{proof}
	The non-Hermitian matrix $\bm G$ in Eq.~\eqref{Incre} can be reformulated as
	\begin{equation*}\label{muSnew}
	\bm G =\underbrace{ \left[                
		\begin{array}{ccc}   
		\lambda \bm K^{\dag} & -\bm I\\
		\bm 0_{m \times m} & \lambda \bm K^{\dag} \\ 
		\end{array}
		\right]}_{\triangleq \bm G_1}  + \underbrace{ \left[   
		\begin{array}{ccc}   
		\bm 0_{m \times m} &  \bm 0_{m \times m} \\
		-\bm y\bm y^{\!\top}/m^3r^2 &  \bm 0_{m \times m} \\ 
		\end{array}
		\right]}_{\triangleq \bm G_2}   \,.
	\end{equation*}
	As a result, $\bm G$ can be represented as a sum of a block upper triangular matrix $\bm G_1$ with a non-Hermitian perturbation $\bm G_2$.
	
	To estimate $\bm G_1$, by Lemma~\ref{henrici}, from the definition of Schur decomposition on $\bm G_1$, it can be easily verified that $\bm D$ and $\bm U$ are
	\begin{equation*}
	\begin{split}
	& \bm D = \diag \! \Big( \frac{\lambda}{\sigma_1}, \dots, \frac{\lambda}{\sigma_p}, 0, \dots,0, \frac{\lambda}{\sigma_{m-q+1}}, \dots, \frac{\lambda}{\sigma_{m}}, \frac{\lambda}{\sigma_1}, \dots, \frac{\lambda}{\sigma_p}, 0, \dots,0, \frac{\lambda}{\sigma_{m-p+1}}, \dots, \frac{\lambda}{\sigma_{m}} \!\Big), \\&~\mbox{and}~ 
	~ \bm U  = \left[        
	\begin{array}{ccc}   
	\bm 0_{m} & -\bm I \\
	\bm 0_{m} & \bm 0_{m} \\ 
	\end{array}
	\right] .
	\end{split}
	\end{equation*}
	Accordingly, $\bm U$ is a nilpotent matrix with $\bm U^2=0$, and thus we have $b=2$.
	According to Lemma~\ref{henrici}, there exists an eigenvalue of $\bm G_1$ denoting as $\sigma(\bm G_1)$ such that
	\begin{equation}\label{muvar}
	\left| \mu -\sigma(\bm G_1) \right| \leq \max(\varsigma, \sqrt[^b]{\varsigma} ) \leq \varsigma + \sqrt[^b]{\varsigma} \,,
	\end{equation}
	where $\varsigma$ is given by
	\begin{equation*}
	\varsigma \coloneqq  \| \bm G_2 \|_2 \sum_{i=1}^{b-1} \big\| \bm U \big\|_2^i = \| \bm G_2 \|_2 \| \bm U \|_2 =\| \bm G_2 \|_2 = \frac{\| \bm y \|^2}{m^3 r^2}\,.
	\end{equation*}
	Then we consider the following three cases based on the sign of $\sigma(\bm G_1)$.
	\begin{case}
		$\sigma(\bm G_1) = 0$\\
		The inequality in Eq.~\eqref{muvar} can be formulated as
		\begin{equation}\label{muvar2}
		- \frac{\| \bm y \|_2}{m \sqrt{m} r} - \frac{\| \bm y \|_2^2}{m^3 r^2} \leq \mu \leq  \frac{\| \bm y \|_2}{m \sqrt{m} r} + \frac{\| \bm y \|_2^2}{m^3 r^2}  \,.
		\end{equation}
	\end{case}
	
	\begin{case}
		$\sigma(\bm G_1) > 0$ \\
		Without loss of generality, we assume that $\sigma(\bm G_1) $ is $\lambda / \sigma_l$ with $l \in \{ 1,2,\cdots,p \}$.
		According to the definition of condition number $C_1$, we have
		\begin{equation*}
		0 < \frac{1}{\sigma_1} \leq \frac{1}{\sigma_l} \leq \frac{C_1}{c_1} m^{-\eta_1} \leq \frac{C_1}{c_1} m^{-\eta},~~ \mbox{$\eta \!=\! \min\{ \eta_1, \eta_2 \} $} \,.
		\end{equation*}
		Then, the inequality in Eq.~\eqref{muvar} can be formulated as
		\begin{equation}\label{muvarneg}
		- \frac{\| \bm y \|_2}{m \sqrt{m} r} - \frac{\| \bm y \|_2^2}{m^3 r^2} \leq \mu \leq  \frac{C_1}{c_1} m^{-(\gamma+\eta)} + \frac{\| \bm y \|_2}{m \sqrt{m} r} + \frac{\| \bm y \|_2^2}{m^3 r^2}  \,.
		\end{equation}
	\end{case}
	
	\begin{case}
		$\sigma(\bm G_1) < 0$ \\
		Likewise, we assume that $\sigma(\bm G_1) $ is $\lambda / \sigma_l$ with $l \in \{ m-q+1, m-q+2,\cdots,m \} $.
		According to the definition of condition number $C_2$, we have
		\begin{equation*}
		0 > \frac{1}{\sigma_m} \geq \frac{1}{\sigma_l} \geq \frac{C_2}{c_m} m^{-\eta_2} \geq \frac{C_2}{c_m} m^{-\eta},~~\mbox{$\eta = \min\{ \eta_1, \eta_2 \} $} \,.
		\end{equation*}
		Then, the inequality in Eq.~\eqref{muvar} can be formulated as
		\begin{equation}\label{muvarpog}
		\frac{C_2}{c_m} m^{-(\gamma+\eta)}  - \frac{\| \bm y \|_2}{m \sqrt{m} r} - \frac{\| \bm y \|_2^2}{m^3 r^2} \leq {\mu} \leq  \frac{\| \bm y \|_2}{m \sqrt{m} r} + \frac{\| \bm y \|_2^2}{m^3 r^2}.
		\end{equation}
	\end{case}	
	Combining Eq.~\eqref{muvar2}, Eq.~\eqref{muvarneg} and Eq.~\eqref{muvarpog}, we have
	\begin{equation*}
	\left\{
	\begin{array}{rcl}
	\begin{split}
	& \mu \geq \frac{C_2}{c_m} m^{-(\gamma+\eta)} - \frac{\| \bm y \|_2}{m \sqrt{m} r} - \frac{\| \bm y \|_2^2}{m^3 r^2} \\
	& \mu \leq \frac{C_1}{c_1} m^{-(\gamma+\eta)}  +  \frac{\| \bm y \|_2}{m \sqrt{m} r} + \frac{\| \bm y \|_2^2}{m^3 r^2}\,,
	\end{split}
	\end{array}\right.
	\end{equation*}	
	which can be further written as
	\begin{equation*}
	\frac{C_2}{c_m} m^{-(\gamma + \eta)} + \tilde{\mu} - \tilde{\mu}^2 \leq {\mu} \leq \frac{C_1}{c_1} m^{-(\gamma + \eta)} - \tilde{\mu} + \tilde{\mu}^2 \,.
	\end{equation*}
	Therefore, we have $\lim_{m \rightarrow \infty} {\mu} = 0$, and its convergence rate is $\mathcal{O}(1/m)$ due to $\gamma + \eta > 1$.
	Finally, ${\mu}$ can be represented in Eq.~\eqref{mu1for} with $\widetilde{c_a} \neq 0$, which concludes the proof.
\end{proof}

\subsubsection{Proofs of Proposition~\ref{propos2pz}}
\label{ratehypo}
Given the expression of $\mu$ with the convergence rate $\mathcal{O}(1/m)$ in Proposition~\ref{promu}, we are ready to present the estimates for $P_2(\bm z, \lambda)$ and $P(\bm z, \lambda)$ as demonstrated by Proposition~\ref{propos2pz}.

\begin{proof}[Proof of Proposition~\ref{propos2pz}]
	{	We cast the proof in two steps: firstly prove the \emph{consistency}, i.e., $\lim_{m \rightarrow \infty} P(\bm z, \lambda) = 0$, and then derive its convergence rate.}
	
	\noindent {\bf Step 1: Consistency of $P(\bm z, \lambda)$}\\
	Based on the decomposition of the hypothesis error $P(\bm z, \lambda)$ in Proposition~\ref{prohyto}, due to $P(\bm z, \lambda) \geq 0$ for any $m \in \mathbb{N}$, we have $\lim_{m \rightarrow \infty} \big(P_1(\bm z, \lambda) + P_2(\bm z, \lambda)\big) \geq 0$ if the limits $\lim_{m \rightarrow \infty} P_1(\bm z, \lambda) $ and  $\lim_{m \rightarrow \infty} P_2(\bm z, \lambda) $ exist.
	Next we analyse $P_1(\bm z, \lambda) $ and $P_2(\bm z, \lambda) $, respectively.
	
	According to the expression of $P_1(\bm z, \lambda)$, it can be bounded by
	\begin{equation}\label{P1zlambdabound}
	\begin{split}
	P_1(\bm z, \lambda) &= \frac{2}{m} \bm y^{\!\top}\bm K \widetilde{\bm \alpha_{\bm z,\lambda}} = -\frac{2}{m^2 \tilde{\mu}} \bm y^{\!\top} \bm K \bm K^{\dag} \bm y \\
	&=  -\frac{2}{m^2 \tilde{\mu}} \bm y^{\! \top} \underbrace{\left( \sum_{i=1}^{p} \bm v_i \bm v_i^{\!\top} + \sum_{i=m-q+1}^{m} \bm v_i \bm v_i^{\!\top}  \right)}_{\triangleq \bm \Xi} \bm y \\
	&\leq \frac{2 \| \bm y \|_2 r}{\sqrt{m}} \,,
	\end{split}
	\end{equation}
	where $\bm v_i$ is the $i$-th column of the orthogonal matrix $\bm V$ from the eigenvalue decomposition $\bm K=\bm V \bm \Sigma \bm V^{\!\top}$.
	The inequality in the above equation holds by $\bm y^{\!\top} \bm \Xi \bm y = \bm y^{\!\top} (\bm I - \sum_{i=p+1}^{m-q}  \bm v_i \bm v_i^{\!\top}) \bm y \leq \bm y^{\!\top} \bm y$.
	
	According to the expression of $P_2(\bm z, \lambda)$, it can be rewritten as
	\begin{equation*}\label{p2f}
	\begin{split}
	P_2(\bm z, \lambda)  &= -\frac{2}{m^2} \bm y^{\!\top} \bm K (\lambda \bm I - {\mu} \bm K)^{\dag}  \bm y = \frac{2}{m^2} \bm y^{\!\top} \left( \sum_{i=1}^{p}  \frac{-\bm v_i \bm v_i^{\!\top}}{\frac{\lambda}{\sigma_i}-{\mu}} + \sum_{i=m-q+1}^{m}  \frac{-\bm v_i \bm v_i^{\!\top}}{\frac{\lambda}{\sigma_i}-{\mu}} \right) \bm y \,.
	\end{split}
	\end{equation*}
	Since the function $h(\sigma_i) = \frac{-1}{\frac{\lambda}{\sigma_i}-{\mu}}$ is an increasing function of $\sigma_i$, $P_2(\bm z, \lambda)$ can be bounded by
	\begin{equation}\label{P2zlambdabound}
	-\frac{2}{m^2} \cdot \frac{1}{\frac{\lambda}{\sigma_1}-{\mu}} \bm y^{\! \top} \bm \Xi \bm y \! \leq \! P_2(\bm z, \lambda) \! \leq \! -\frac{2}{m^2} \cdot \frac{1}{\frac{\lambda}{\sigma_{m-q+1}}-{\mu}} \bm y^{\! \top} \bm \Xi \bm y\,.
	\end{equation}
	By Proposition~\ref{promu}, plugging Eq.~\eqref{mu1for} into the above inequality, when $\eta + \gamma > 1$, we have
	\begin{equation*}
	\begin{split}
	\lim_{ m \rightarrow \infty} -\frac{2}{m^2} \cdot \frac{1}{\frac{\lambda}{\sigma_1}-{\mu}} \bm y^{\! \top} \bm \Xi \bm y 
	& = \lim_{ m \rightarrow \infty} -\frac{2}{m^2} \cdot \frac{1}{\frac{\lambda}{\sigma_{m-q+1}}-{\mu}} \bm y^{\! \top} \bm \Xi \bm y
	= \lim_{ m \rightarrow \infty} \frac{2 \bm y^{\! \top} \bm \Xi \bm y }{\sqrt{m}\| \bm y \|} \! \cdot \! \frac{r}{-\widetilde{c_a}} \\
	& \leq \lim_{ m \rightarrow \infty} \frac{2\| \bm y \|_2r}{-\widetilde{c_a}\sqrt{m}}< \infty \,,
	\end{split}
	\end{equation*}
	which holds by $\| \bm y \|_2 = \mathcal{O}(\sqrt{m})$ and $\widetilde{c_a} \neq 0$.
	According to the squeeze theorem, we conclude that the limit $\lim_{m \rightarrow \infty}P_2(\bm z, \lambda)$ exists.
	Because of $ P(\bm z, \lambda) \geq 0$, we have
	\begin{equation*}
	\begin{split}
	0 & \leq \lim_{m \rightarrow \infty} \Big(P_1(\bm z, \lambda) + P_2(\bm z, \lambda)\Big)
	\leq   \lim_{m \rightarrow \infty} \left[ \frac{2\| \bm y \|_2r}{ \sqrt{m}} \left(1-\frac{1}{\widetilde{c_a}} \right) \right] \,,
	\end{split}
	\end{equation*}
	which indicates that $1-\frac{1}{\widetilde{c_a}} \geq 0 $, i.e., $\widetilde{c_a} \geq 1$.
	Accordingly, the coefficient in Eq.~\eqref{mu1for} $\widetilde{c_a} \in [-1, 0) \bigcup (0,1]$ can be further improved to $\widetilde{c_a} =1$.
	In this case, it is obvious that $\lim_{m \rightarrow \infty} \Big(P_1(\bm z, \lambda) + P_2(\bm z, \lambda)\Big) = 0$ implies the consistency for $P(\bm z, \lambda)$.
	
	\noindent{\bf Step 2: Convergence rate of $P(\bm z, \lambda)$}\\
	Based on the consistency of $P(\bm z, \lambda)$, we derive its convergence rate as follows.
	For notational simplicity, we denote
	$\widetilde{c_e} \coloneqq  \left[ \frac{C_2}{c_m} + \widetilde{c_d} \left( \frac{C_1}{c_1} -  \frac{C_2}{c_m} \right) \right] $.
	Accordingly, by virtue of Eqs.~\eqref{P1zlambdabound},\eqref{P2zlambdabound} and Proposition~\ref{promu} for $\mu$, we have
	\begin{equation*}
	\begin{split}
	P(\bm z, \lambda) & = P_1(\bm z, \lambda) + P_2(\bm z, \lambda) \\
	& \leq \frac{2 \| \bm y \|_2 r}{\sqrt{m}} + \frac{2\| \bm y \|_2^2}{m^2} \cdot \frac{1}{\frac{\lambda}{\sigma_{m-q+1}}-{\mu}} \\
	& \leq \frac{2 \| \bm y \|_2 r}{\sqrt{m}} + \frac{2\| \bm y \|_2^2}{m}  \frac{1}{\frac{-1}{\sigma_{m-q+1}}m^{1-\gamma} - \frac{\| \bm y \|_2}{\sqrt{m}r} - \frac{ \widetilde{c_b} \| \bm y \|_2^2}{m^2 r^2} - \widetilde{c_e}  m^{-\gamma}} \\
	& \leq  \frac{2 \| \bm y \|_2 r}{\sqrt{m}} + \frac{2\| \bm y \|_2^2}{m} \Big( - \frac{\sqrt{m}r}{\| \bm y \|} \frac{-c_m}{C_2} m^{1-\gamma-\eta} + \frac{\| \bm y \|_2^2}{mr^2} m^{-1} + |\widetilde{c_e}| m^{-(\gamma+\eta)} \Big)   \\
	&  \leq \left( 2Mr + 2M^2 \Big( \frac{-c_m}{C_2} + \frac{M^2}{r^2} + \frac{C_1}{c_1} \Big)  \right) m^{-\Theta_1} \\
	& \triangleq \widetilde{C_1} m^{-\Theta_1} \,,
	\end{split}
	\end{equation*}
	where $\widetilde{C_1} \coloneqq  2Mr + 2M^2 \big( \frac{-c_m}{C_2} + \frac{M^2}{r^2} + \frac{C_1}{c_1} \big) $ and the power index is
	$ \Theta_1 = \min \big\{ 1, \gamma+\eta-1 \big\}$.
	Finally, we conclude the proof for Proposition~\ref{propos2pz}.
\end{proof}

\subsection{Estimate Sample Error}
\label{sec:se}
The sample error can be decomposed into $ S(\bm z, \lambda)  =  S_1(\bm z, \lambda) +  S_2(\bm z, \lambda)$ with
\begin{equation*}
\begin{split}
& S_1(\bm z, \lambda) = \mathcal{E}\big(\pi_B(f_{\bm{z},\lambda})\big) -  \mathcal{E}(f_{\rho}) - \mathcal{E}_{\bm{z}}\big(\pi_B(f_{\bm{z},\lambda})\big) \! +  \! \mathcal{E}_{\bm{z}}(f_{\rho}) \,,\\
& S_2(\bm z, \lambda) = \Big\{ \mathcal{E}_{\bm{z}}\big(f_{\lambda}\big) - \mathcal{E}_{\bm{z}}(f_{\rho}) \Big\} - \Big\{ \mathcal{E}(f_{\lambda}) - \mathcal{E}(f_{\rho}) \Big\}\,.
\end{split}
\end{equation*}
Note that $S_1(\bm z, \lambda)$ involves the samples $\bm z$.
Thus a uniform concentration inequality for a family of functions containing $f_{\bm z, \lambda}$ is needed to estimate $S_1(\bm z, \lambda)$.
Since we have $f_{\bm z, \lambda} \in \mathcal{B}_R$ defined by Eq.~\eqref{BRradius}, we shall bound $S_1$ by the following proposition with a properly chosen $R$.
{Considering that the estimates for $S_1(\bm z, \lambda)$ and $S_2(\bm z, \lambda)$ have been extensively investigated in \cite{Wu2006Learning,cucker2007learning,Shi2014Quantile},
	we directly present the corresponding results in Appendix~\ref{sec:appsamp} under the existence of $f_{\rho}$ in Assumption~\ref{assrho}, and the regularity condition on $\rho$ in Assumption~\ref{asscap}.}

\subsection{Derive Learning Rates}
\label{sec:learn}

Combining the bounds in Proposition~\ref{errdec},~\ref{propos2pz} and estimates for the sample error, the excess error $ \mathcal{E}\big(\pi_B(f_{\bm{z},\lambda})\big) - \mathcal{E}(f_{\rho})$ can be estimated.
Specifically, as aforementioned, algorithmically, the radius $r$ or $R$ in Eq.~\eqref{BRradius} is determined by cross validation in our experiments.
Theoretically, in our analysis, it is estimated by giving a bound for $\lambda \langle f_{\bm{z},\lambda} ,Tf_{\bm{z},\lambda} \rangle_{\mathcal{H_K}}$.
This is conducted by the iteration technique \cite{Wu2006Learning} to improve learning rates.
Under Assumption~\ref{assrho} to \ref{eigenassump}, the proof for learning rates in Theorem~\ref{theorem1ls} can be found in Appendix~\ref{sec:applearnrate}.

\section{Numerical Experiments}
\label{sec:exp}
In this section, we validate our theoretical results by numerical experiments in the following three aspects.

\subsection{Eigenvalue assumption} 
Here we verify the justification of our eigenvalue decay assumption in Assumption~\ref{eigenassump} on four indefinite kernels, including
\begin{itemize}
	\item the spherical polynomial (SP) kernel \cite{pennington2015spherical}: $k_p(\bm x, \bm x') = (1+ \langle \bm x, \bm x' \rangle)^p$ with $p=10$ on the unit sphere is shift-invariant but indefinite.
	\item the \textit{TL1} kernel \cite{huang2017classification}:  $k_{\tau'}(\bm{x},\bm{x}')= \max\{\tau'-\|\bm{x}-\bm{x}' \|_1,0\}$ with $\tau' = 0.7d$ as suggested.
	\item the Delta-Gauss kernel \cite{oglic18a}: It is formulated as the difference of two Gaussian kernels, i.e., $k\left(\bm x, \bm x' \right)=\exp \left(-\left\|\bm x - \bm x' \right\|^{2}/ \tau_{1} \right)-\exp \left(-\left\|\bm x - \bm x' \right\|^{2}/ \tau_{2} \right)$ with $\tau_1 = 1$ and $\tau_2=0.1$.
	\item the \textit{log} kernel \cite{boughorbel2005conditionally}: $k(\bm x,\bm x') = - \log(1+\| \bm x - \bm x'\|)$.
\end{itemize}
Here the Delta-Gaussian kernel \cite{oglic18a} and the log kernel \cite{boughorbel2005conditionally} are associated with RKKS while the SP and \verb"TL1" kernels have not been proved as reproducing kernels in RKKS. It is still an open problem to verify that a kernel admits the decomposition \cite{liu2020fast}.
The Delta-Gaussian kernel is defined as the difference of two Gaussian kernels, and thus it is clear that $\sigma_1$ and $\sigma_m$ follow with the exponential decay in the same rate, i.e., $\eta_1 = \eta_2$. 
For the \verb"log" kernel \cite{boughorbel2005conditionally} is a conditionally positive definite kernel of order one\footnote{The order in conditionally positive definite kernels is an important concept, refer to \cite{wendland2004scattered} for details.} associated with RKKS.
According to Theorem 8.5 in \cite{wendland2004scattered}, the kernel matrix induced by this kernel has only one negative eigenvalue.
Further, we can conclude that the only one negative eigenvalue admits $\sigma_m = -\sum_{i=1}^{m-1} \sigma_i$ because of $k(\bm 0)=\frac{1}{n}\mathrm{tr}(\bm K)= \frac{1}{n}\sum_{i=1}^n \lambda_i = 0$, which implies $\eta_2 > \eta_1$.

Figure~\ref{eigkernels} experimentally shows eigenvalue distributions of the above four indefinite kernels on the \emph{monks3} dataset\footnote{\url{https://archive.ics.uci.edu/ml/datasets.html}}.
It can be found that our eigenvalue assumption:  $\sigma_1 \geq c_1 m^{\eta_1}$ ($c_1>0$, $\eta_1 > 0$) and $\sigma_m \leq c_m m^{\eta_2}$ ($c_m<0$, $\eta_2>0$) in Definition~\ref{eigenassump} is reasonable.
Specifically, our experiments on the \verb"log" kernel verify that it has only one negative eigenvalue admitting $\sigma_m = -\sum_{i=1}^{m-1} \sigma_i$.
Note that although the SP and \verb"TL1" kernels have not been proved as reproducing kernels in RKKS, our eigenvalue assumption still covers them, which demonstrates the feasibility of our assumption.

\begin{figure*}[t]
	\centering
	\subfigure[\emph{SP kernel}]{\label{spk}
		\includegraphics[width=0.32\textwidth]{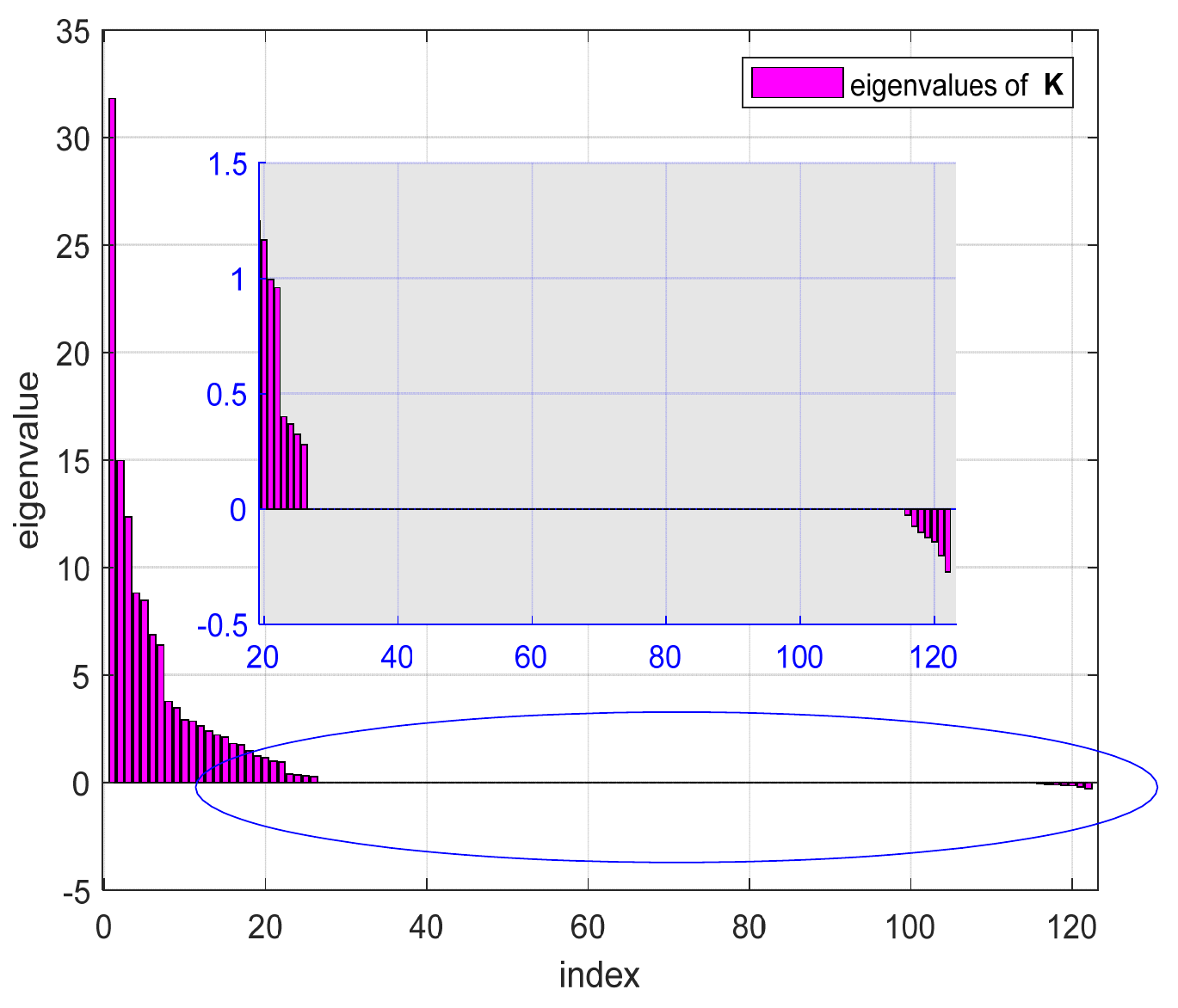}}
	\hspace{0.5cm}
	\subfigure[\emph{TL1 kernel}]{\label{tl1acc}
		\includegraphics[width=0.34\textwidth]{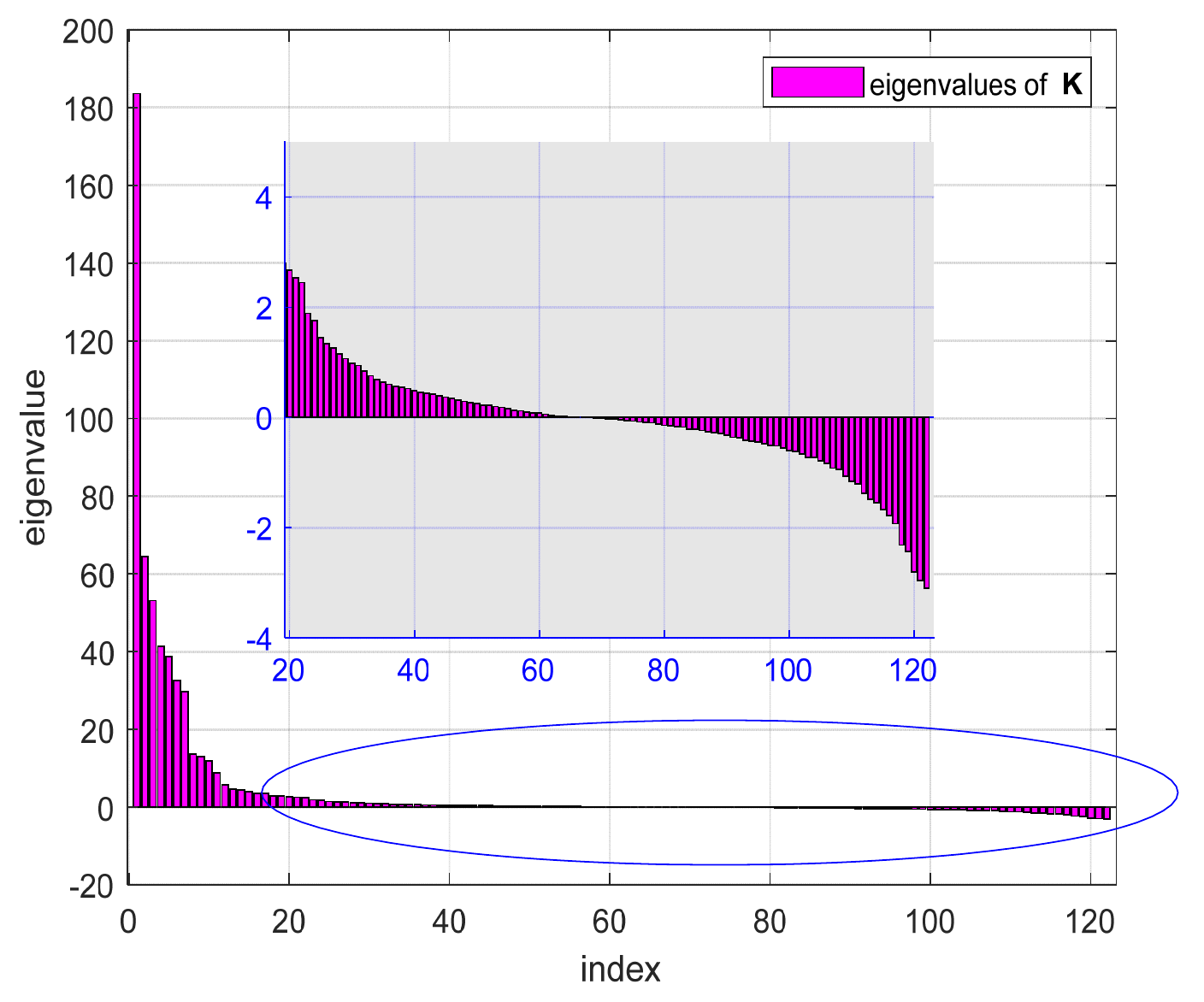}}
	\subfigure[\emph{Delta-Gauss kernel}]{\label{deltg}
		\includegraphics[width=0.32\textwidth]{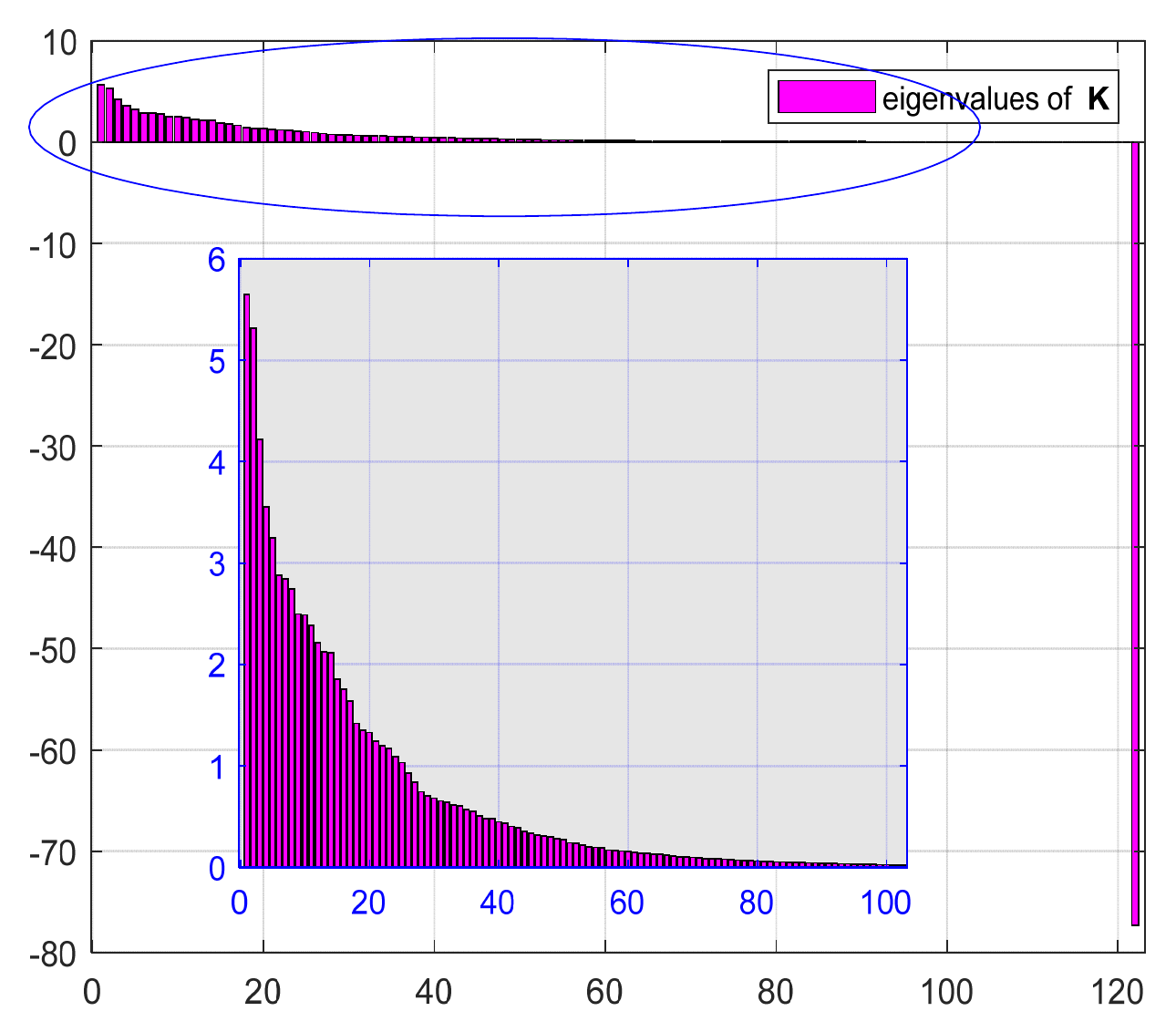}}
	\hspace{0.5cm}
	\subfigure[\emph{log kernel}]{\label{logk}
		\includegraphics[width=0.32\textwidth]{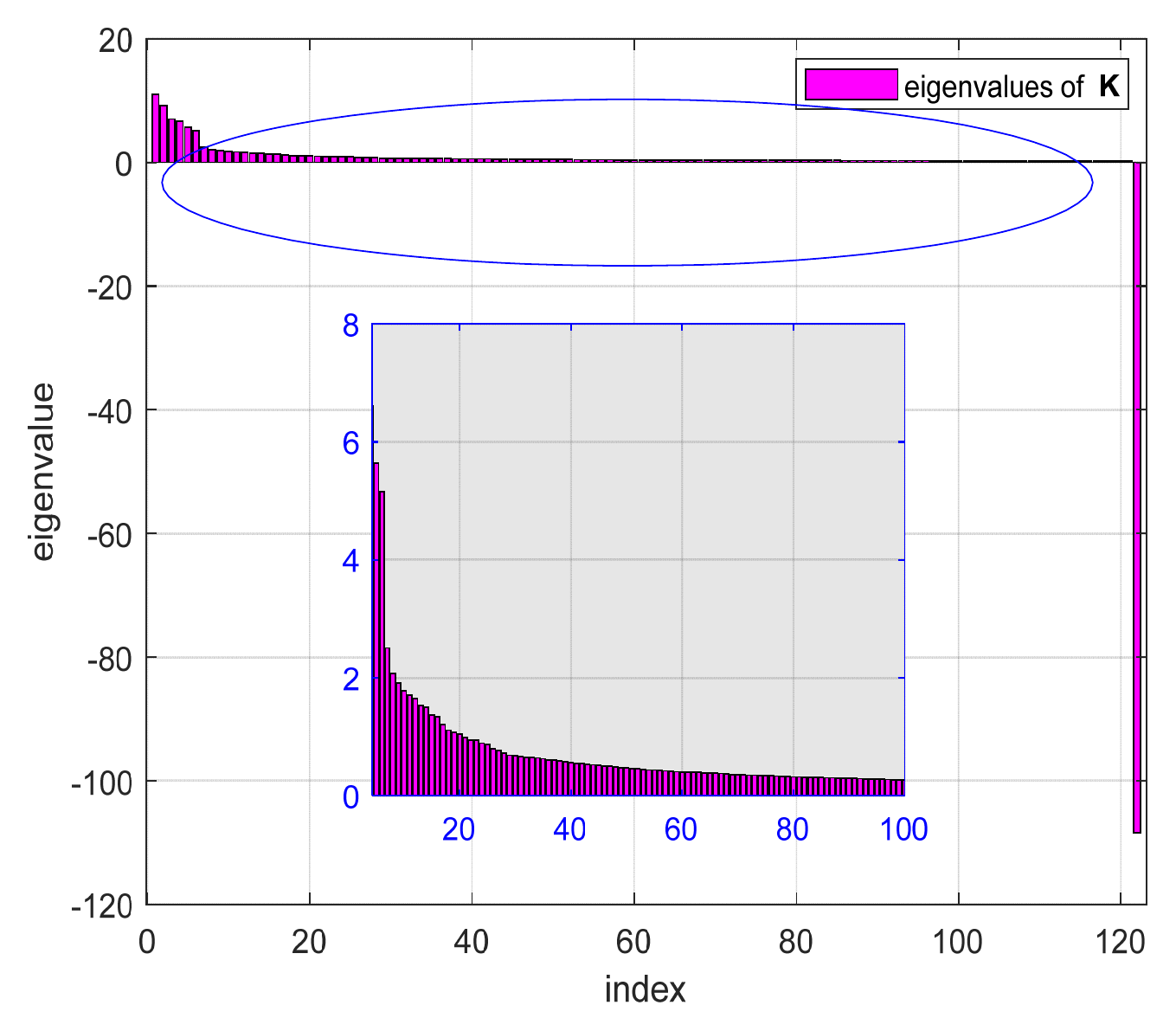}}
	\caption{Eigenvalue distribution of kernel matrices generated by various indefinite kernels on the \emph{monks3} dataset.}	\label{eigkernels}
\end{figure*}

\begin{figure}[t]
	\centering
	\subfigure[\emph{Delta-Gauss kernel}]{\label{errgauss}
		\includegraphics[width=0.32\textwidth]{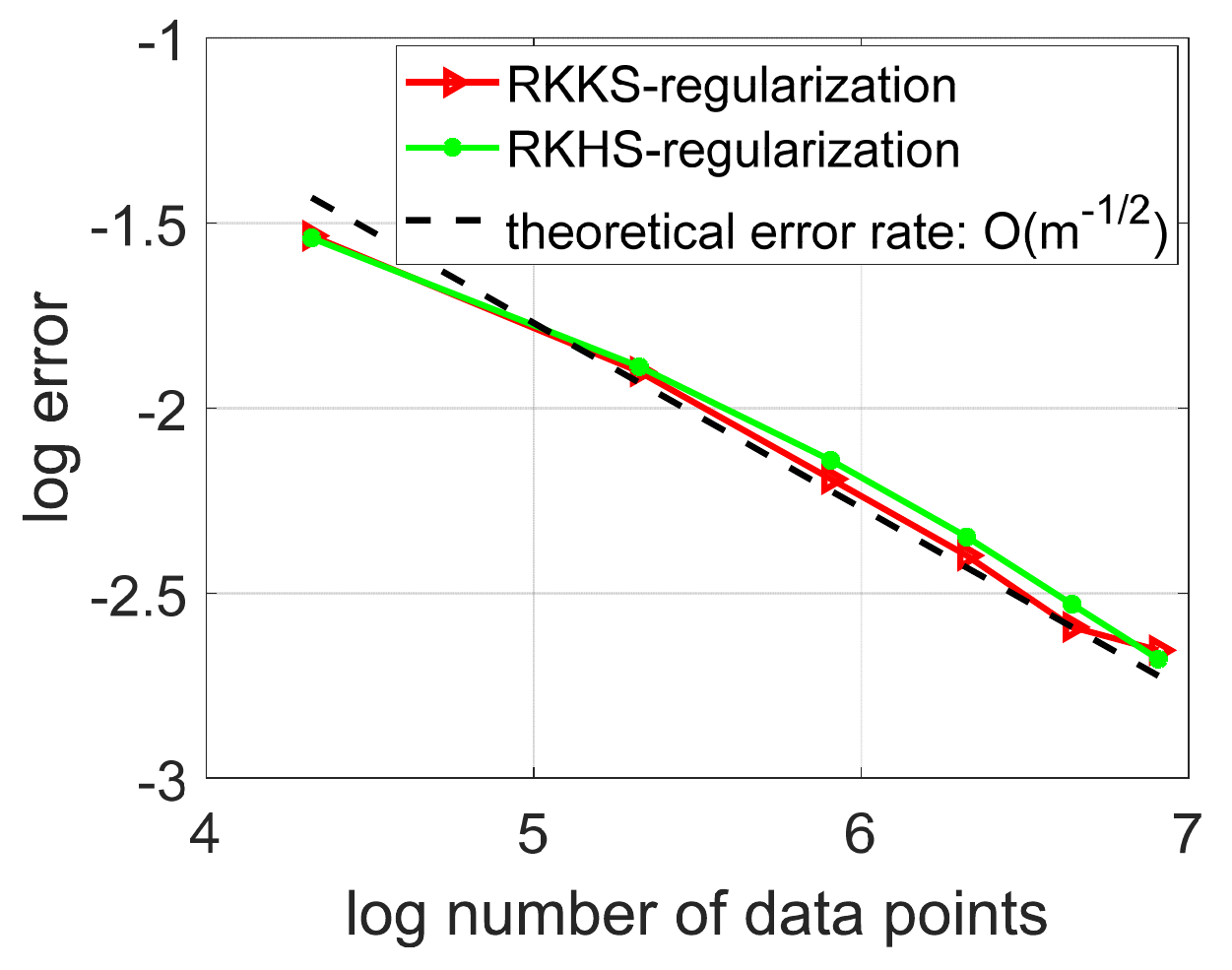}}
	\hspace{0.5cm}
	\subfigure[\emph{log kernel}]{\label{errlog}
		\includegraphics[width=0.32\textwidth]{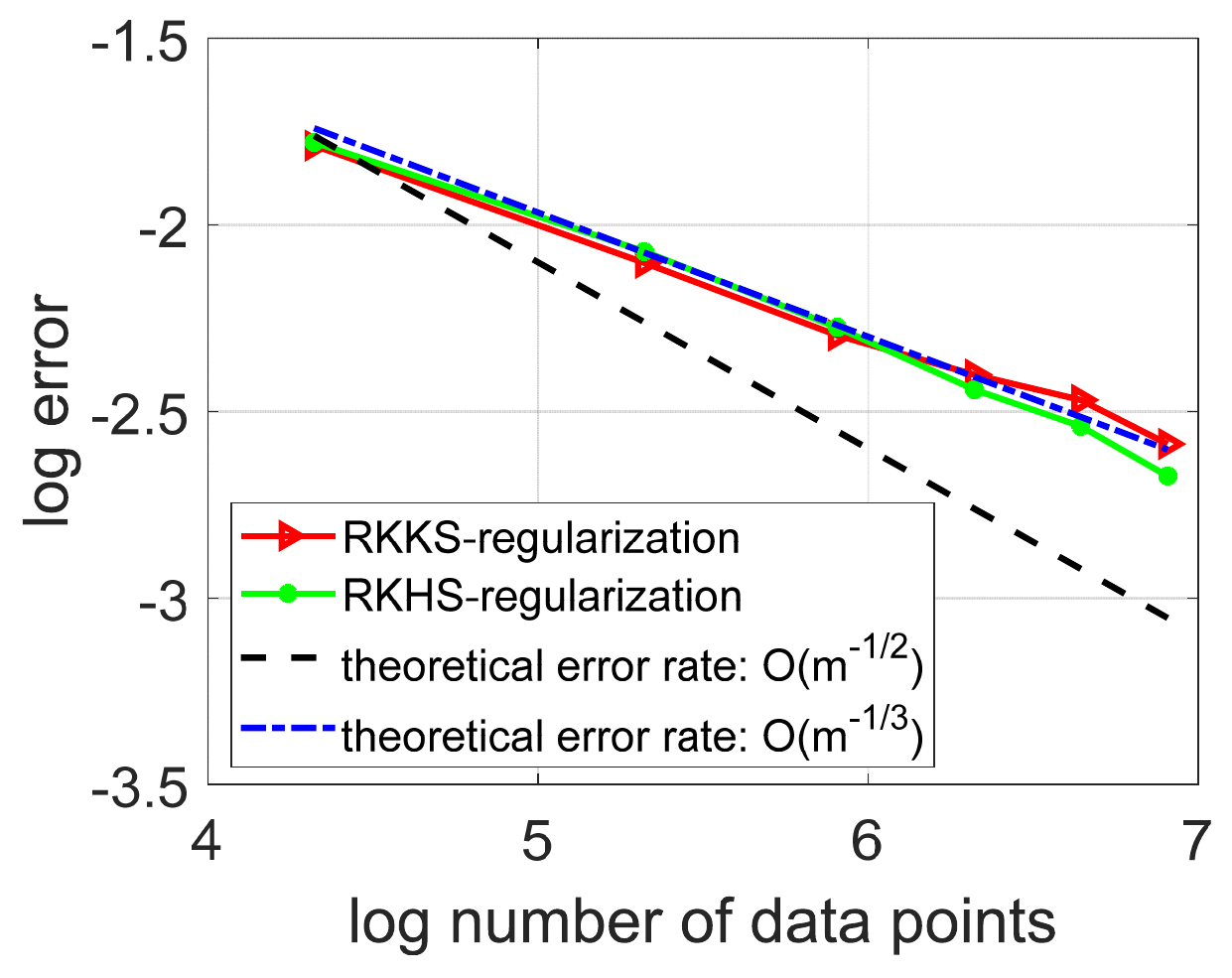}}
	\caption{The log-log plot of the theoretical and observed risk convergence rates averaged on 100 trials.}\label{err-rate}
	\vspace{-0.4cm}
\end{figure}

\subsection{Empirical validations of derived learning rates}
Here we verify the derived convergence rates on the \emph{monks3} dataset effected by different indefinite kernels.
In our experiment, we choose $\lambda \coloneqq  1/m$ and two indefinite kernels including the Delta-Gauss kernel and the \verb"log" kernel on \emph{monks3} to study in what degree they would effect the learning rates.
Since the selected two kernels are $C^{\infty}(X \times X)$, $s$ can be arbitrarily small.
In this case, by Theorem~\ref{theorem1ls} and Corollary~\ref{corolr}, the learning rate of problem~\eqref{fzrs} with the RKKS regularizer $\langle f,f\rangle_{\mathcal{H_K}}$ or the RKHS regularizer $\| f\|_{{\mathcal{H_{\bar K}}}}^2$ is close to $\min \{ \beta, \eta\}$.
Here the two parameters $\beta$ and $\eta$ indicate the approximation ability for $f_{\rho}$ and the size of RKKS by different indefinite kernels, and thus they will influence the expected risk rate.
Figure~\ref{errgauss} shows the observed learning rate associated with the Delta-Gauss kernel is $\mathcal{O}(1/\sqrt{m})$, while the excess risk associated with the \verb"log" kernel converges at $\mathcal{O}(m^{-1/3})$ in Figure~\ref{errlog}.
Hence, Figure~\ref{err-rate} demonstrates this difference that the excess risk of problem~\eqref{fzrs} with the Delta-Gauss kernel converges faster than that with the \verb"log" kernel.
This is reasonable and demonstrated by Theorem~\ref{theorem1ls}, i.e., different $\mathcal{H_K}$ spanned by various indefinite kernels lead to different convergence rates due to their different approximation ability for $f_{\rho}$.

The above experiments validate the rationality of our eigenvalue assumption and the consistency with theoretical results.

\section{Conclusion}
\label{sec:conclusion}
In this paper, we provide approximation analysis of the least squares problem associated with the $\langle f,f\rangle_{\mathcal{H_K}}$ regularization scheme in RKKS.
For this non-convex problem with the bounded hyper-sphere constraint, we can get an attainable optimal solution, which makes it possible to conduct approximation analysis in RKKS.
Accordingly, we start the analysis from the learning problem that has an analytical solution, and thus obtain the first-step to understand the learning behavior in RKKS.
Our analysis and experimental validation bridge the gap between the regularized risk minimization problem in RKHS and RKKS.

\section*{Acknowledgement}
The research leading to these results has received funding from the European Research Council under the European Union's Horizon 2020 research and innovation program / ERC Advanced Grant E-DUALITY (787960). This paper reflects only the authors' views and the Union is not liable for any use that may be made of the contained information.
This work was supported in part by Research Council KU Leuven: Optimization frameworks for deep kernel machines C14/18/068; Flemish Government: FWO projects: GOA4917N (Deep Restricted Kernel Machines: Methods and Foundations), PhD/Postdoc grant. This research received funding from the Flemish Government (AI Research Program). 
This work was supported in part by Ford KU Leuven Research Alliance Project KUL0076 (Stability analysis and performance improvement of deep reinforcement learning algorithms), EU H2020 ICT-48 Network TAILOR (Foundations of Trustworthy AI - Integrating Reasoning, Learning and Optimization), Leuven.AI Institute; and in part by the National Natural Science Foundation of China (Grants No. 61572315, 61876107, 11631015, 61977046), in part by the National Key Research and Development Project (No. 2018AAA0100702),in part by Program of Shanghai Subject Chief Scientist (Project No.18XD1400700). Fanghui Liu and Lei Shi contributed equally to this work.

\newcommand{\etalchar}[1]{$^{#1}$}
\providecommand{\bysame}{\leavevmode\hbox to3em{\hrulefill}\thinspace}
\providecommand{\MR}{\relax\ifhmode\unskip\space\fi MR }
\providecommand{\MRhref}[2]{%
	\href{http://www.ams.org/mathscinet-getitem?mr=#1}{#2}
}
\providecommand{\href}[2]{#2}

\clearpage

\appendix

\section{Proof for the Sample Error}
\label{sec:appsamp}
The asymptotic behaviors of $S_1(\bm z, \lambda)$ and $S_2(\bm z, \lambda)$ are usually illustrated by
the convergence of the empirical mean $\frac{1}{m}\sum_{i=1}^m\xi_{i}$
to its expectation $\mathbb{E}\xi$, where
$\left\{\xi_{i}\right\}_{i=1}^m$ are independent random variables on
$(Z,\rho)$ defined as
\begin{equation*}\label{randomvariables}
\xi(\bm x, y) \coloneqq \big(y-f_{\lambda}(\bm x)\big)^2 - \big(y-f_{\rho}(\bm x)\big)^2\,.
\end{equation*}
For $R\geq 1$, denote
\begin{equation*}\label{set}
\mathscr{W}(R)=\left\{{\bm z}\in Z^{m}: \sqrt{\langle f_{\bm{z},\lambda} ,Tf_{\bm{z},\lambda} \rangle_{\mathcal{H_K}}} \leq R\right\}.
\end{equation*}

\begin{lemma}\label{lemma4}
	If $\xi$ is a symmetric real-valued function on $X \times Y$ with mean $\mathbb{E}(\xi)$. Assume that $\mathbb{E}(\xi) \geq 0$, $|\xi - \mathbb{E}\xi|\leq T$ almost surely and $\mathbb{E}\xi^2 \leq c'_1 (\mathbb{E}\xi)^{\theta}$ for some $0 \leq \theta \leq 1$ and $c'_1 \geq 0$, $T \geq 0$.
	Then for every $\epsilon>0$ there holds
	\begin{equation*}\label{leamm4eq}
	\mathop{\rm Prob} \left\{
	\frac{\frac{1}{m}\sum_{i=1}^m
		\xi(\bm z_i) \!-\! \mathbb{E}\xi}{{\sqrt{(\mathbb{E}\xi)^{\theta}+\epsilon^{\theta}}}}  \geq
	\epsilon^{1-\frac{\theta}{2}}  \right\}  \leq
	\exp \!   \left\{ \frac{-m\epsilon^{2-\theta}}{2c'_1+\frac{2}{3}T \epsilon^{1-\theta}} \!\right\}\!.
	\end{equation*}
\end{lemma}

Now we can bound $S_2(\bm z, \lambda)$ by the following proposition.
\begin{proposition}\label{propos2}
	Suppose that $|f_{\rho}(\bm x)|\leq M^*$ with $M^* \geq 1$, for any $0 < \delta < 1$, there exists a subset of $Z_1$ of $Z^{m}$
	with confidence at least $1-\delta/2$, such that for any $\forall {\bm z} \in Z_1$
	\begin{equation*}\label{bound7}
	S_2(\bm z, \lambda)
	\leq \frac{1}{2}{D}(\lambda)
	+ \frac{1}{m} \bigg(\kappa \sqrt{\frac{D(\lambda)}{\lambda}} + M^*+ 12\bigg) \log\frac{2}{\delta}\,.
	\end{equation*}
\end{proposition}

\begin{proof}
	From the definition of $f_{\lambda}$ in Proposition~\ref{errdec}, combining Eq.~\eqref{fnormdiff} and Eq.~\eqref{Dlamdadef}, we have
	\begin{equation}\label{flambound}
	\| f_{\lambda}\|_{\infty} \leq \kappa \sqrt{\langle f_{\lambda} ,Tf_{\lambda} \rangle_{\mathcal{H_K}}} \leq \kappa \sqrt{\frac{D(\lambda)}{\lambda}} \leq \kappa \sqrt{C_0}\lambda^{\frac{\beta-1}{2}}\,,
	\end{equation}
	which leads to $\|f_{\lambda} \|_{\infty} \leq \kappa \sqrt{\frac{D(\lambda)}{\lambda}}$.
	The first equality holds because the reproducing kernel $k_+  + k_-$ associated with ${\mathcal{H_{\bar K}}}$ is the square root of the limiting kernel in \cite{Guo2017Optimal} associated with the empirical covariance operator $T$.
	Due to $f_{\rho}(\bm x)$ contained in $[-M^*,M^*]$, we can get
	\begin{equation*}
	\big|\xi - \mathbb{E}(\xi) \big| \leq   \kappa \sqrt{\frac{D(\lambda)}{\lambda}}+M^*\,.
	\end{equation*}
	For least squared loss, $\mathbb{E}(\xi^2) \leq 4 \mathbb{E}(\xi)$ indicates $c'_1=4$ and $\theta = 1$.
	Applying Lemma \ref{lemma4}, there exists a subset $Z_1$ of $Z^m$ with confidence $1-\delta/2$, we have
	\begin{equation*}\label{bound6}
	\begin{split}
	\frac{1}{m} \sum_{i=1}^{m}\xi(\bm z_i) - \mathbb{E} \xi \leq
	\sqrt{(\mathbb{E}\xi)^{\theta}+\epsilon^{\theta}}
	\epsilon^{1-\frac{\theta}{2}} \leq \frac{1}{2}\mathbb{E}\xi +\frac{3}{2}\epsilon,
	\end{split}
	\end{equation*}
	Then, we obtain
	\begin{equation*}
	\begin{split}
	\frac{1}{m}\sum_{i=1}^m \xi(\bm z_{i})-\mathbb{E}\xi
	&\leq
	\frac{\theta}{2} \! \Big\{ \mathcal{E}(f_{\lambda}) \!-\! \mathcal{E}(f_{\rho}) \Big\}
	\!+\! \frac{ T + 3c'_1}{m} \log\frac{2}{\delta} \\
	& \leq \frac{1}{2}{D}(\lambda)
	\!+\! \frac{\kappa \sqrt{\frac{D(\lambda)}{\lambda}} \!+\! M^*\!+\! 12}{m} \log\frac{2}{\delta}\,,
	\end{split}
	\end{equation*}
	which concludes the proof.
\end{proof}

In the next, we attempt to bound $S_1(\bm z, \lambda)$ with respect to the samples $\bm z$.
Thus a uniform concentration inequality for a family of functions containing $f_{\bm z, \lambda}$ is needed to estimate $S_1$.
Since we have $f_{\bm z, \lambda} \in \mathcal{B}_R$, which is defined by Eq.~\eqref{BRradius}, we shall bound $S_1$ by the following proposition with a properly chosen $R$.
\begin{proposition}\label{propos1}
	Suppose that $|f_{\rho}(\bm x)|\leq M^*$ with $M^* \geq 1$ in Assumption~\ref{assrho}, and $\rho$ satisfies the regularity condition in Assumption~\ref{asscap}, for any $0 < \delta < 1$, $R \geq 1$, $B > 0$, there exists a subset $Z_2$ of $Z^{m}$
	with confidence at least $1-\delta/2$, such that for any $\bm z \in \mathscr{W}(R) \cap Z_2$,
	\begin{equation*}\label{boundS1}
	\begin{split}
	S_1(\bm z, \lambda)
	&\leq \! \frac{136(M^*+B)}{m}\log\frac{2}{\delta} \!+\!  \frac{1}{2}\Big\{ \! \mathcal{E}\big(\pi_B(f_{\bm{z},\lambda})\big)
	\!-\! \mathcal{E}(f_{\rho}) \! \Big\} + 144C_s(M^*+B)m^{-\frac{1}{1+s}}R^{\frac{s}{1+s}}  \,.
	\end{split}
	\end{equation*}
\end{proposition}

\begin{proof}
	Consider the function set $\mathcal{F}_R$ with $R>0$ by
	\begin{equation*}
	\mathcal{F}_R \coloneqq  \left\{ \big(y-\pi_B(f)(\bm x)\big)^2 - \big(y-f_{\rho}(\bm x)\big)^2 : f \in \mathcal{B}_R \right\}\,.
	\end{equation*}
	We can easily see that each function $g \in \mathcal{F}_R$ satisfies $\| g\|_{\infty} \leq B+M^*$, and thus we have $| g-\mathbb{E}g| \leq B+M^*$.
	So using $\mathscr{N}(\mathcal{F}_R,\epsilon) \leq \mathscr{N}(\mathcal{B}_1,\epsilon)$ and applying Lemma~\ref{lemma4} to the function set $\mathcal{F}_R$ with the covering number condition in Eq.~\eqref{assumpN} in Assumption~\ref{asscap}, we have
	\begin{equation*}
	\begin{split}
	&\mathop{\mathrm{Prob}} \limits_{\bm z \in Z^{m}} \Bigg\{ \sup_{f \in \mathcal{F}_R}  \frac{\mathbb{E}g - \frac{1}{m} \sum_{i=1}^{m} g(\bm x_i, y_{i})  }{\sqrt{(\mathbb{E}g)^{\theta} +\epsilon^{\theta}}} \geq 4\epsilon^{1-\frac{\theta}{2}} \Bigg\} 
	\leq \exp \left\{ C_s \Big( \frac{R}{\epsilon} \Big)^s -\frac{m \epsilon^{2-\theta}}{2c'_1 +\frac{2}{3}(B+M^*) \epsilon^{1-\theta}} \right\}\,,
	\end{split}
	\end{equation*}
	with $\mathbb{E}g = \mathcal{E}\big(\pi_B(f)\big) - \mathcal{E}(f_{\rho})$.
	Hence there holds a subset $Z_2$ of $Z^{m}$ with confidence at least $1-\delta/2$ such that $\forall \bm z \in Z_2 \cap \mathscr{W}(R)$
	\begin{equation*}
	\sup_{f \in \mathcal{F}_R}  \frac{\mathbb{E}g - \frac{1}{m} \sum_{i=1}^{m}  g( \bm x_i, y_{i})  }{\sqrt{(\mathbb{E}g)^{\theta} +\Big(\epsilon^*(m,R,\frac{\delta}{2})\Big)^{\theta}}} \leq  4\big(\epsilon^*(m,R,\frac{\delta}{2})\big)^{1-\frac{\theta}{2}} \,,
	\end{equation*}
	where $\epsilon^*(m,R,\frac{\delta}{2})$ is the smallest positive number $\epsilon$ satisfying
	\begin{equation*}
	C_s \Big( \frac{R}{\epsilon} \Big)^s -\frac{m \epsilon^{2-\theta}}{2c'_1 +\frac{2}{3}(M^*+B) \epsilon^{1-\theta}} = \log \frac{\delta}{2}\,,
	\end{equation*}
	using Lemma 7.2 in \cite{cucker2007learning}, we have
	\begin{equation*}
	\begin{split}
	\epsilon^* & \leq \max\Bigg\{\frac{48+2(M^*+B)}{3m}\log\frac{2}{\delta}, \left( \frac{48+4(B+M^*)}{3m}C_sR^s \right)^{\frac{1}{1+s}} \Bigg\} \\
	& \leq \frac{17(M^*+B)}{m}\log\frac{2}{\delta} + 18C_s(M^*+B)m^{-\frac{1}{1+s}}R^{\frac{s}{1+s}}\,,
	\end{split}
	\end{equation*}
	where we use $M^*\geq 1$.
	For $\bm z \in \mathcal{B}(R) \cap Z_2$, we have
	\begin{equation*}
	\begin{split}
	S_1(\bm z, \lambda)
	\leq 8 \epsilon^*\big(m,R,\frac{\delta}{2}\big) +  \frac{1}{2}\Big\{ \mathcal{E}\big(\pi_B(f_{\bm{z},\lambda})\big) - \mathcal{E}(f_{\rho}) \Big\}  \,.
	\end{split}
	\end{equation*}
\end{proof}

\section{Proof for Learning Rates}
\label{sec:applearnrate}
Combining the bounds in Proposition~\ref{errdec},~\ref{propos2pz},~\ref{propos2},~\ref{propos1}, and Eq.~\eqref{flambound},
let Eq.~\eqref{assumpN} with $s>0$, Eq.~\eqref{Dlambda} with $0 < \beta \leq 1$, take $\lambda=m^{-\gamma}$ with $ 0 < \gamma < 1$, the excess error $ \mathcal{E}\big(\pi_B(f_{\bm{z},\lambda})\big) - \mathcal{E}(f_{\rho})$ can be bounded by
\begin{equation}\label{finalbf}
\begin{split}
\mathcal{E}\big(\pi_B(f_{\bm{z},\lambda})\big) - \mathcal{E}(f_{\rho}) +\lambda \langle f_{\bm{z},\lambda} ,Tf_{\bm{z},\lambda} \rangle_{\mathcal{H_K}} 
& \leq 3C_0m^{-\gamma \beta}  
+\widetilde{C_1} m ^{- \Theta_1} 
+\widetilde{C_2}\log\frac{2}{\delta}m^{-1}  \\& + \widetilde{C_3} m^{-\frac{1}{1+s}}R^{\frac{s}{1+s}} \log\frac{2}{\delta}  + 2\kappa \sqrt{C_0}m^{-\big( \frac{\gamma(\beta-1)}{2}+1\big)}\log\frac{2}{\delta} \,,
\end{split}
\end{equation}
where $\widetilde{C_1}$ is given in Proposition~\ref{propos2pz}.
Two constants $\widetilde{C_2}$ and $\widetilde{C_3}$ are given by
\begin{eqnarray*}
	\widetilde{C_2} = 274M^*+272B+24, \quad
	\widetilde{C_3} = 288(M^*+B) C_s\,.
\end{eqnarray*}

In the next, we attempt to find a $R>0$ by giving a bound for $\lambda \langle f_{\bm{z},\lambda} ,Tf_{\bm{z},\lambda} \rangle_{\mathcal{H_K}}$.
\begin{lemma}\label{lemmafzr}
	Suppose that $\rho$ satisfies the condition in Eq.~\eqref{Dlambda} with $0 < \beta \leq 1$ in Assumption~\ref{assreg}.
	For some $s>0$ in Assumption~\ref{asscap},  take $\lambda=m^{-\gamma}$ with $ 0 < \gamma \leq 1$. Then for $0 < \epsilon < 1$ and $0<\delta<1$ with confidence $1-\delta$, we have
	\begin{equation}\label{fzrR}
	\sqrt{\langle f_{\bm{z},\lambda} ,Tf_{\bm{z},\lambda} \rangle_{\mathcal{H_K}}} \leq 4\widetilde{C_3} \widetilde{C_{X}} \left(\log\frac{2}{\epsilon}\right)^2 \sqrt{\log\frac{2}{\delta}} m^{\theta_{\epsilon}}\,,
	\end{equation}
	where $\widetilde{C_{X}}$ is given by
	\begin{equation*}
	\widetilde{C_{X}} =  \left(1+\sqrt{\widetilde{C_2}  } + \sqrt{2\kappa \sqrt{C_0} } + \sqrt{3C_0} + \sqrt{\widetilde{C_1}} \right) \,,
	\end{equation*}
	and $\theta_{\epsilon}$ is
	\begin{equation}\label{thetaep}
	\theta_{\epsilon} \!=\!  \max \left\{ \!  \frac{\gamma(1-\beta)}{2}, \! \frac{1 - \eta }{2}, \! (\gamma(1+s)-1)(2+s)+\epsilon \! \right\} \,.
	\end{equation}
\end{lemma}

\begin{proof}
	According to Eq.~\eqref{finalbf}, we
	know that for any $R \geq 1$ there exists a subset $V_R$ of $Z_m$ with measure at most $\delta$ such that
	\begin{equation*}
	\sqrt{\langle f_{\bm{z},\lambda} ,Tf_{\bm{z},\lambda} \rangle_{\mathcal{H_K}}} \leq a_m R^{\frac{s}{2+2s}} +b_m, \quad \forall \bm z \in \mathscr{W}(R) \backslash V_R\,,
	\end{equation*}
	where $a_m = \sqrt{ \widetilde{C_3}} m^{\frac{\gamma}{2} - \frac{1}{2(1+s)}}$, and $b_m$ is defined as
	\begin{equation*}
	b_m \!=\! \left( \sqrt{\widetilde{C_2} \log\frac{2}{\delta} } \!+\! \sqrt{2\kappa \sqrt{C_0} \log\frac{2}{\delta}} \!+\! \sqrt{3C_0} \!+\! \sqrt{\widetilde{C_1}} \right) \! m^{\zeta},
	\end{equation*}
	where the power index $\zeta$ is
	\begin{equation*}
	\begin{split}
	\zeta & = \max \left\{  \frac{\gamma(1-\beta)}{2}, \frac{\gamma-1}{2}, \frac{\gamma}{2} - \frac{\gamma (\beta - 1)+2}{4},  \frac{1 - \eta }{2}  \right\} \\
	& = \max \left\{  \frac{\gamma(1-\beta)}{2},  \frac{1 - \eta }{2}  \right\} \,.
	\end{split}
	\end{equation*}
	It tells us that $\mathscr{W}(R) \subseteq \mathscr{W} \left( a_mR^{\frac{s}{2+2s}} + b_m  \right) \bigcup V_R$.
	Define a sequence $\{ R^{(j)} \}_{j=0}^J$ with $ R^{(j)}  = a_m (R^{(j-1)})^{s/(2+2s)} + b_m$ with $J \in \mathbb{N}$, we have $ Z^m = \mathscr{W}(R^{(0)})$ satisfying
	\begin{equation*}
	\mathscr{W}(\!R^{(0)}) \! \subseteq \! \mathscr{W}(\!R^{(1)})  \bigcup  V_{R^{(0)}} \! \subseteq \! \cdots \! \subseteq \! \mathscr{W}(R^{(J)}) \! \bigcup \! \left(\! \bigcup_{j=0}^{J-1} \!\! V_{R^{(j)}} \!\!\! \right)\!\!.
	\end{equation*}
	Since each set $V_{R^{(j)}}$ is at most $\delta$, the set $\mathscr{W}(R^{(J)})$ has measure at least $1-J \delta$.
	
	Denote $\Delta = s/(2+2s) < 1/2$, the definition of the sequence $\{ R^{(j)} \}_{j=0}^J$ indicates that
	\begin{equation*}
	R^{(J)} \!=\! \underbrace{ a_m^{1+\Delta+\cdots+\Delta^{J-1}}\! (R^{(0)})^{\Delta^J} }_{R^{(J)}_1} \!\!+\!\! \underbrace{ \sum_{j=1}^{J-1}\! a_m^{1+\Delta+\cdots+\Delta^{j-1}} b_m^{\Delta^j} \!+\! b_m }_{R^{(J)}_2}\!.
	\end{equation*}
	The first term $R^{(J)}_1$ can be bounded by
	\begin{equation*}
	R^{(J)}_1 \leq \widetilde{C_3} m^{(\gamma(1+s)-1)(2+s)} m^{\frac{1}{1+s}2^{-J}}\,,
	\end{equation*}
	where $J$ is chosen to be the smallest integer satisfying $ J \geq \frac{\log(1/ \epsilon) }{ \log2}$.
	Besides, $R^{(J)}_2$ can be bounded by
	\begin{equation*}
	R^{(J)}_2 \! \leq \! m^{(\gamma(1+s)-1)(2+s)}  \widetilde{C_3} b_1 \!\! \sum_{j=0}^{J-1} \! m^{\!\big(\! \zeta - (\gamma(1+s)-1)(2+s)  \! \big) \! \frac{s^j}{(2+2s)^j} }\!,
	\end{equation*}
	with $b_1\coloneqq \sqrt{\widetilde{C_2} \log\frac{2}{\delta} } + \sqrt{2\kappa \sqrt{C_0} \log\frac{2}{\delta}} + \sqrt{3C_0} + \sqrt{\widetilde{C_1}}$.
	When $\zeta \leq (\gamma(1+s)-1)(2+s) $, $R^{(J)}_2$ can be bounded by $\widetilde{C_3} b_1 J m^{(\gamma(1+s)-1)(2+s)}$.
	When $\zeta > (\gamma(1+s)-1)(2+s) $, $R^{(J)}_2$ can be bounded by $\widetilde{C_3} b_1 J m^{\zeta}$.
	Based on the above discussion, we have
	\begin{equation*}
	R^{(J)} \leq (\widetilde{C_3} + \widetilde{C_3} b_1 J) m^{\theta_{\epsilon}}\,,
	\end{equation*}
	with $\theta_{\epsilon}=\max \{ \zeta, (\gamma(1+s)-1)(2+s)+\epsilon  \}$.
	So with confidence $1-J\delta$, there holds
	\begin{equation*}
	\sqrt{\langle f_{\bm{z},\lambda} ,Tf_{\bm{z},\lambda} \rangle_{\mathcal{H_K}}} \leq R^{(J)} \leq \widetilde{C_3} \widetilde{C_{X}}  J \sqrt{\log\frac{2}{\delta}} m^{\theta_{\epsilon}}\,,
	\end{equation*}
	which follows by replacing $\delta$ by $\delta / J$ and noting $J \leq 2 \log(2/ \epsilon)$.
	Finally, we conclude the proof.
\end{proof}

Now, by Lemma~\ref{lemmafzr} and Eq.~\eqref{finalbf}, we are able to prove our main result in Theorem~\ref{theorem1ls}.
\begin{proof}
	Take $R$ to be the right hand side of Eq.~\eqref{fzrR} by Lemma~\ref{lemmafzr}, there exists a subset $V'_R$ of $Z_m$ with measure at most $\delta$ such that $Z^m / V'_R \subseteq \mathscr{W}(R)$.
	Therefore, there exists another subset $V_R$ of $Z^m$ with measure at most $\delta$ such that for any $\bm z \in \mathscr{W}(R) / V_R$, Eq.~\eqref{finalbf} can be formulated as
	\begin{equation*}
	\begin{split}
	\mathcal{E}\big(\pi_B(f_{\bm{z},\lambda})\big) - \mathcal{E}(f_{\rho}) & \leq 3C_0m^{-\gamma \beta} +\widetilde{C_1} m ^{- \Theta_1} +\widetilde{C_2}\log\frac{2}{\delta}m^{-1}
	+ 2\kappa \sqrt{C_0}m^{-\big( \frac{\gamma(\beta-1)}{2}+1\big)}\log\frac{2}{\delta} \\
	& ~~~~+ \widetilde{C_4} \left(\log\frac{2}{\epsilon} \right)^2 \sqrt{\log\frac{2}{\delta}} m^{\frac{s\theta_{\epsilon} - 1}{1+s}} \,,
	\end{split}
	\end{equation*}
	where $\widetilde{C_4} = \widetilde{C_X} (4\widetilde{C_3})^{\frac{s}{1+s}}$.
	Accordingly, by setting the constant $\widetilde{C}$ with
	\begin{equation*}
	\widetilde{C} = 3C_0 +\widetilde{C_1}+ \widetilde{C_2} +  2\kappa \sqrt{C_0} +  \widetilde{C_4} \,,
	\end{equation*}
	we have the following error bound
	\begin{equation*}
	\begin{split}
	\big\| \pi_{M^*}(f_{\bm{z},\lambda})  - f_{\rho} \big\|^2_{L_{\rho_X}^{2}}  \leq  \widetilde{C}  \left(\log\frac{2}{\epsilon} \right)^2 \log\frac{2}{\delta} m^{- \Theta } \,,
	\end{split}
	\end{equation*}
	with confidence $1-\delta$ and the power index $\Theta$ is
	\begin{equation}\label{thetav}
	\Theta = \min \left\{ \gamma \beta, \gamma + \eta -1, \frac{1 - s \theta_{\epsilon} }{1+s} \right\}\,,
	\end{equation}
	provided that $\theta_{\epsilon} < 1 / s$.
	Combining Eq.~\eqref{thetaep} and Eq.~\eqref{thetav}, when $0 < \eta < 1$, we have
	\begin{equation*}
	\begin{split}
	\Theta & =\! \min \! \bigg\{ \gamma \beta, \gamma + \eta -1, \frac{2 - s \gamma (1 -\beta) }{2(1+s)},
	\frac{2-s(1-\eta)}{2(1+s)}, \frac{1-s(\gamma(1+s)-1)(2+s)-s\epsilon}{1+s} \bigg\},
	\end{split}
	\end{equation*}
	where $\epsilon$ is given by Eq.~\eqref{epst} and $\eta$ needs to be further restricted by $\max\{0, 1-2/s\} < \eta < 1$.
	These two restrictions ensure that $\Theta$ is positive for a valid learning rate.
	Specifically, when $\eta \geq 1$, the power index $\Theta$ can be simplified as
	\begin{equation*}
	\begin{split}
	\Theta & = \min \bigg\{ \gamma \beta, \frac{2 - s \gamma (1 -\beta) }{2(1+s)}, \frac{1-s(\gamma(1+s)-1)(2+s)-s\epsilon}{1+s} \bigg\}\,,
	\end{split}
	\end{equation*}
	which concludes the proof.
\end{proof}

\newcommand{\etalchar}[1]{$^{#1}$}
\providecommand{\bysame}{\leavevmode\hbox to3em{\hrulefill}\thinspace}
\providecommand{\MR}{\relax\ifhmode\unskip\space\fi MR }
\providecommand{\MRhref}[2]{%
  \href{http://www.ams.org/mathscinet-getitem?mr=#1}{#2}
}
\providecommand{\href}[2]{#2}
\begin{thebibliography}{SVGDB{\etalchar{+}}02}

\bibitem[ACGZ16]{Alabdulmohsin2016Large}
Ibrahim Alabdulmohsin, Moustapha Cisse, Xin Gao, and Xiangliang Zhang,
  \emph{Large margin classification with indefinite similarities}, Machine
  Learning \textbf{103} (2016), no.~2, 215--237.

\bibitem[AN17]{Adachi2017Eigenvalue}
Satoru Adachi and Yuji Nakatsukasa, \emph{Eigenvalue-based algorithm and
  analysis for nonconvex \protect{QCQP} with one constraint}, Mathematical
  Programming (2017), no.~1, 1--38.

\bibitem[And09]{ando2009projections}
Tsuyoshi Ando, \emph{Projections in krein spaces}, Linear Algebra and its
  Applications \textbf{431} (2009), no.~12, 2346--2358.

\bibitem[Bac13]{bach2013sharp}
Francis Bach, \emph{Sharp analysis of low-rank kernel matrix approximations},
  Proceedings of Conference on Learning Theory, 2013, pp.~185--209.

\bibitem[Bog74]{bognar1974indefinite}
J{\'a}nos Bogn{\'a}r, \emph{Indefinite inner product spaces}, Springer, 1974.

\bibitem[BTB05]{boughorbel2005conditionally}
Sabri Boughorbel, J-P Tarel, and Nozha Boujemaa, \emph{Conditionally positive
  definite kernels for \protect{SVM} based image recognition}, Proceedings of
  IEEE International Conference on Multimedia and Expo, 2005, pp.~113--116.

\bibitem[CDPB19]{cho2019large}
Hyunghoon Cho, Benjamin DeMeo, Jian Peng, and Bonnie Berger, \emph{Large-margin
  classification in hyperbolic space}, Proceedings of International Conference
  on Artificial Intelligence and Statistics, PMLR, 2019, pp.~1832--1840.

\bibitem[Chu86]{Chu1986Generalization}
King Wah~Eric Chu, \emph{Generalization of the \protect{B}auer-\protect{F}ike
  theorem}, Numerische Mathematik \textbf{49} (1986), no.~6, 685--691.

\bibitem[CWYZ04]{Chen2004Support}
Dirong Chen, Qiang Wu, Yiming Ying, and Dingxuan Zhou, \emph{Support vector
  machine soft margin classifiers: error analysis}, Journal of Machine Learning
  Research \textbf{5} (2004), no.~3, 1143--1175.

\bibitem[CZ07]{cucker2007learning}
Felipe Cucker and Dingxuan Zhou, \emph{Learning theory: an approximation theory
  viewpoint}, vol.~24, Cambridge University Press, 2007.

\bibitem[DGK04]{dhillon2004kernel}
Inderjit~S Dhillon, Yuqiang Guan, and Brian Kulis, \emph{Kernel k-means:
  spectral clustering and normalized cuts}, Proceedings of ACM SIGKDD
  international conference on Knowledge discovery and data mining, ACM, 2004,
  pp.~551--556.

\bibitem[FS19]{farooq2019learning}
Muhammad Farooq and Ingo Steinwart, \emph{Learning rates for kernel-based
  expectile regression}, Machine Learning \textbf{108} (2019), no.~2, 203--227.

\bibitem[GGM88]{Gander1988A}
Walter Gander, Gene~H. Golub, and Urs~Von Matt, \emph{A constrained eigenvalue
  problem}, Linear Algebra and Its Applications \textbf{114-115} (1988),
  815--839.

\bibitem[GMR{\etalchar{+}}15]{gao2015minimax}
Chao Gao, Zongming Ma, Zhao Ren, Harrison~H. Zhou, et~al., \emph{Minimax
  estimation in sparse canonical correlation analysis}, The Annals of
  Statistics \textbf{43} (2015), no.~5, 2168--2197.

\bibitem[GS19]{Guo2017Optimal}
Zheng-Chu Guo and Lei Shi, \emph{Optimal rates for coefficient-based
  regularized regression}, Applied and Computational Harmonic Analysis
  \textbf{47} (2019), no.~3, 662--701.

\bibitem[HSW{\etalchar{+}}18]{huang2017classification}
Xiaolin Huang, Johan~A.K. Suykens, Shuning Wang, Joachim Hornegger, and Andreas
  Maier, \emph{Classification with truncated $\ell_1$ distance kernel}, IEEE
  Transactions on Neural Networks and Learning Systems \textbf{29} (2018),
  no.~5, 2025 -- 2030.

\bibitem[HW03]{hoffman2003variation}
Alan~J. Hoffman and Helmut~W. Wielandt, \emph{The variation of the spectrum of
  a normal matrix}, Selected Papers Of Alan J Hoffman: With Commentary, World
  Scientific, 2003, pp.~118--120.

\bibitem[JCO19]{jun2019kernel}
Kwang-Sung Jun, Ashok Cutkosky, and Francesco Orabona, \emph{Kernel truncated
  randomized ridge regression: Optimal rates and low noise acceleration},
  Proceedings of Advances in Neural Information Processing Systems, 2019,
  pp.~15358--15367.

\bibitem[Lan62]{langer1962spektraltheoriej}
Heinz Langer, \emph{Zur spektraltheoriej-selbstadjungierter operatoren},
  Mathematische Annalen \textbf{146} (1962), no.~1, 60--85.

\bibitem[LCC16]{Ga2016Learning}
Ga\"{e}lle Loosli, St\'{e}phane Canu, and Soon~Ong Cheng, \emph{Learning
  \protect{SVM} in \protect{K}re\u{\i}n spaces}, IEEE Transactions on Pattern
  Analysis and Machine Intelligence \textbf{38} (2016), no.~6, 1204--1216.

\bibitem[LGZ17]{lin2017distributed}
Shao-Bo Lin, Xin Guo, and Ding-Xuan Zhou, \emph{Distributed learning with
  regularized least squares}, Journal of Machine Learning Research \textbf{18}
  (2017), no.~1, 3202--3232.

\bibitem[LHCS20]{liu2020fast}
Fanghui Liu, Xiaolin Huang, Yingyi Chen, and Johan~A.K. Suykens, \emph{Fast
  learning in reproducing kernel {Kre\u{\i}n} spaces via generalized measures},
  arXiv preprint arXiv:2006.00247 (2020).

\bibitem[LHG{\etalchar{+}}20]{liu2020learning}
Fanghui Liu, Xiaolin Huang, Chen Gong, Jie Yang, and Li~Li, \emph{Learning
  data-adaptive non-parametric kernels}, Journal of Machine Learning Research
  \textbf{21} (2020), no.~208, 1--39.

\bibitem[LZL20]{liu2020simplemkkm}
Xinwang Liu, En~Zhu, and Jiyuan Liu, \emph{\protect{SimpleMKKM}: Simple
  multiple kernel k-means}, arXiv preprint arXiv:2005.04975 (2020).

\bibitem[OG18]{oglic18a}
Dino Oglic and Thomas G\"{a}ertner, \emph{Learning in reproducing kernel
  \protect{K}re\u{\i}n spaces}, Proceedings of the International Conference on
  Machine Learning, 2018, pp.~3859--3867.

\bibitem[OG19]{oglic2019scalable}
Dino Oglic and Thomas G{\"a}rtner, \emph{Scalable learning in reproducing
  kernel kre\u{\i}n spaces}, Proceedings of International Conference on Machine
  Learning, 2019, pp.~4912--4921.

\bibitem[OMS04]{Cheng2004Learning}
Cheng~Soon Ong, Xavier Mary, and Alexander~J. Smola, \emph{Learning with
  non-positive kernels}, Proceedings of the International Conference on Machine
  Learning, 2004, pp.~81--89.

\bibitem[OSW05]{Cheng2005Learning}
Cheng~Soon Ong, Alexander~J. Smola, and Robert~C Williamson, \emph{Learning the
  kernel with hyperkernels}, Journal of Machine Learning Research \textbf{6}
  (2005), no.~Jul, 1043--1071.

\bibitem[PH09]{Pekalska2009Kernel}
El\.{z}bieta P\c{e}kalska and Bernard Haasdonk, \emph{Kernel discriminant
  analysis for positive definite and indefinite kernels}, IEEE Transactions on
  Pattern Analysis and Machine Intelligence \textbf{31} (2009), no.~6,
  1017--1032.

\bibitem[PYK15]{pennington2015spherical}
Jeffrey Pennington, Felix Xinnan~X. Yu, and Sanjiv Kumar, \emph{Spherical
  random features for polynomial kernels}, Proceedings of Advances in Neural
  Information Processing Systems, 2015, pp.~1846--1854.

\bibitem[RR17]{Rudi2017Generalization}
Alessandro Rudi and Lorenzo Rosasco, \emph{Generalization properties of
  learning with random features}, Proceedings of Advances in Neural Information
  Processing Systems, 2017, pp.~3215--3225.

\bibitem[SA08]{Steinwart2008SVM}
Ingo Steinwart and Christmann Andreas, \emph{Support vector machines}, Springer
  Science and Business Media, 2008.

\bibitem[Sch99]{schaback1999native}
Robert Schaback, \emph{Native \protect{Hilbert} spaces for radial basis
  functions \protect{I}}, New Developments in Approximation Theory, Springer,
  1999, pp.~255--282.

\bibitem[SDSGR18]{sala2018representation}
Frederic Sala, Chris De~Sa, Albert Gu, and Christopher Re, \emph{Representation
  tradeoffs for hyperbolic embeddings}, Proceedings of International Conference
  on Machine Learning, 2018, pp.~4460--4469.

\bibitem[SHFS19]{shi2019sparse}
Lei Shi, Xiaolin Huang, Yunlong Feng, and Johan~AK Suykens, \emph{Sparse kernel
  regression with coefficient-based $\ell_q$- regularization.}, Journal of
  Machine Learning Research \textbf{20} (2019), no.~161, 1--44.

\bibitem[SHS09]{Steinwart2009Optimal}
Ingo Steinwart, Don~R. Hush, and Clint Scovel, \emph{Optimal rates for
  regularized least squares regression}, Proceedings of Conference on Learning
  Theory, 2009, pp.~1--10.

\bibitem[SHTS14]{Shi2014Quantile}
Lei Shi, Xiaolin Huang, Zheng Tian, and Johan~A.K. Suykens, \emph{Quantile
  regression with $\ell_1$-regularization and \protect{G}aussian kernels},
  Advances in Computational Mathematics \textbf{40} (2014), no.~2, 517--551.

\bibitem[SML{\etalchar{+}}19]{shang2019unsupervised}
Ronghua Shang, Yang Meng, Chiyang Liu, Licheng Jiao, Amir M~Ghalamzan Esfahani,
  and Rustam Stolkin, \emph{Unsupervised feature selection based on kernel
  fisher discriminant analysis and regression learning}, Machine Learning
  \textbf{108} (2019), no.~4, 659--686.

\bibitem[SOW01]{smola2001regularization}
Alex~J. Smola, Zoltan~L. Ovari, and Robert~C. Williamson, \emph{Regularization
  with dot-product kernels}, Proceedings of Advances in Neural Information
  Processing Systems, 2001, pp.~308--314.

\bibitem[SP20]{saha2020learning}
Akash Saha and Balamurugan Palaniappan, \emph{Learning with operator-valued
  kernels in reproducing kernel kre\u{\i}n spaces}, Proceedings of Advances in
  Neural Information Processing Systems, 2020, pp.~1--11.

\bibitem[SS90]{stewart1990matrix}
Gilbert~W. Stewart and Jiguang Sun, \emph{Matrix perturbation theory}, Harcourt
  Brace Jovanoich, 1990.

\bibitem[SS03]{Sch2003Learning}
Bernhard Sch\"{o}lkopf and Alexander~J. Smola, \emph{Learning with kernels:
  support vector machines, regularization, optimization, and beyond}, MIT
  Press, 2003.

\bibitem[SS07]{steinwart2007fast}
Ingo Steinwart and Clint Scovel, \emph{Fast rates for support vector machines
  using \protect{G}aussian kernels}, Annals of Statistics \textbf{35} (2007),
  no.~2, 575--607.

\bibitem[ST15]{Schleif2015Indefinite}
Frank~Michael Schleif and Peter Tino, \emph{Indefinite proximity learning: a
  review}, Neural Computation \textbf{27} (2015), no.~10, 2039--2096.

\bibitem[SVGDB{\etalchar{+}}02]{suykens2002least}
Johan~A.K. Suykens, Tony Van~Gestel, Jos De~Brabanter, Bart De~Moor, and Joos
  Vandewalle, \emph{Least squares support vector machines}, World Scientific,
  2002.

\bibitem[TY19]{terada2019kernel}
Yoshikazu Terada and Michio Yamamoto, \emph{Kernel normalized cut: a
  theoretical revisit}, Proceedings of International Conference on Machine
  Learning, 2019, pp.~6206--6214.

\bibitem[Wen04]{wendland2004scattered}
Holger Wendland, \emph{Scattered data approximation}, vol.~17, Cambridge
  university press, 2004.

\bibitem[WYZ06]{Wu2006Learning}
Qiang Wu, Yiming Ying, and Dingxuan Zhou, \emph{Learning rates of least-square
  regularized regression}, Foundations of Computational Mathematics \textbf{6}
  (2006), no.~2, 171--192.

\bibitem[WZ11]{Wang2011Optimal}
Cheng Wang and Ding-Xuan Zhou, \emph{Optimal learning rates for least squares
  regularized regression with unbounded sampling}, Journal of Complexity
  \textbf{27} (2011), no.~1, 55--67.

\bibitem[XWS16]{Xia2016S}
Yong Xia, Shu Wang, and Ruey~Lin Sheu, \emph{S-lemma with equality and its
  applications}, Mathematical Programming \textbf{156} (2016), no.~1-2,
  513--547.

\bibitem[YCG09]{Ying2009Analysis}
Yiming Ying, Colin Campbell, and Mark Girolami, \emph{Analysis of \protect{SVM}
  with indefinite kernels}, Proceedings of Advances in Neural Information
  Processing Systems, 2009, pp.~2205--2213.

\bibitem[ZH02]{Zhu2002Kernel}
Ji~Zhu and Trevor Hastie, \emph{Kernel logistic regression and the import
  vector machine}, Journal of Computational and Graphical Statistics
  \textbf{14} (2002), no.~1, 185--205.

\end{thebibliography}
\end{document}